\documentclass[a4paper]{article}
\usepackage{here}
\usepackage{lipsum}
\usepackage[authoryear]{natbib}
\usepackage{amsmath,amsthm,amssymb}
\usepackage{graphicx}
\usepackage{hhline,url}
\usepackage{authblk}
\usepackage{tikz}
\usepackage{comment}
\usetikzlibrary{positioning}
\usetikzlibrary{calc}

\newcommand{\acite}{\citet}
\renewcommand{\cite}{\citep}

\newtheorem{thm}{Theorem}[section]
\newtheorem{lem}[thm]{Lemma}
\newtheorem{remark}[thm]{Remark}
\newenvironment{rem}{\begin{remark}\rm}{\end{remark}}

\newtheorem{Definition}[thm]{Definition}
\newenvironment{dfn}{\begin{Definition}\rm}{\end{Definition}}
\newtheorem{Corollary}[thm]{Corollary}
\newenvironment{cor}{\begin{Corollary}}{\end{Corollary}}

\newcommand{\R}{\mathbb{R}}
\newcommand{\Z}{\mathbb{Z}}
\newcommand{\ve}{\varepsilon}

\newcommand{\F}{\mathcal{F}}

\newcommand{\N}{\mathcal{N}}
\newcommand{\dd}{\,\mathrm{d}}

\DeclareMathOperator{\conv}{conv}
\DeclareMathOperator{\sgn}{sgn}
\renewcommand{\hat}{\widehat}
\renewcommand{\tilde}{\widetilde}

\renewcommand{\phi}{\varphi}

\title{On the minimax optimality and superiority of deep neural network learning over sparse parameter spaces}                      

\author[1]{Satoshi Hayakawa\footnote{satoshi\_hayakawa@mist.i.u-tokyo.ac.jp}}
\author[1,2]{Taiji Suzuki\footnote{taiji@mist.i.u-tokyo.ac.jp}}

\affil[1]{Graduate School of Information Science and Technology, The University of Tokyo}
\affil[2]{Center for Advanced Intelligence Project, RIKEN}

\date{}

\providecommand{\keywords}[1]{\textbf{\textit{Keywords---}} #1}

\begin{document}

\maketitle

\begin{abstract}
	Deep learning has been applied to various tasks in the field of machine learning
and has shown superiority to other common procedures such as kernel methods.
To provide a better theoretical understanding of the reasons for its success,
we discuss the performance of deep learning and other methods
on a nonparametric regression problem with a Gaussian noise.
Whereas existing theoretical studies of deep learning have been based mainly on
mathematical theories of well-known function classes such as H\"{o}lder and Besov classes,
we focus on function classes with discontinuity and sparsity,
which are those naturally assumed in practice.
To highlight the effectiveness of deep learning, we compare deep learning with a class of linear estimators representative of a class of shallow estimators.
It is shown that the minimax risk of a linear estimator
on the convex hull of a target function class
does not differ from that of the original target function class.
This results in the suboptimality of linear methods
over a simple but non-convex function class,
on which deep learning can attain nearly the minimax-optimal rate.
In addition to this extreme case,
we consider function classes with sparse wavelet coefficients.
On these function classes, deep learning also attains the minimax rate up to log factors of the sample size,
and linear methods are still suboptimal if the assumed sparsity is strong.
We also point out that the parameter sharing of deep neural networks can
remarkably reduce the complexity of the model in our setting.\\
\end{abstract}


\section{Introduction}

Deep learning has been successfully applied to a number of machine learning problems,
including image analysis and speech recognition
\cite{schmidhuber2015deep,goodfellow2016deep}.
However, 
the rapid expansion of its applications has preceded a thorough theoretical understanding,
and thus the theoretical properties of neural networks and their learning 
have not yet been fully understood.
This paper aims to 
summarize recent developments in theoretical analyses of deep learning
and to provide new approximation and estimation error bounds that theoretically
confirm the superiority of deep learning to other representative methods.

In this section,
we present an overview of the paper;
here,
we prioritize understandability 
over strict mathematical rigor.
Some formal definitions and restrictions, such as for
measurability and integrability, are presented in later sections.


\subsection{Nonparametric regression}
Throughout this paper,
our intent is to demonstrate the superiority of the deep learning approach to other methods.
To do so, we consider a simple nonparametric regression problem
and compare the performance of various approaches in that setting.
The nonparametric regression problem we analyze is formulated as follows:
\begin{quote}
	We observe $n$ i.i.d.\ input--output pairs $(X_i, Y_i)\in[0, 1]^d\times\R$
	generated by the model
	\[
		Y_i=f^\circ(X_i)+\xi_i,\quad i=1,\ldots,n,
	\]
	where $\xi_i$ is an i.i.d. noise independent of inputs.
	The object is to estimate $f^\circ$ from the observed data.
\end{quote}
This problem setting has been commonly used in statistical learning theory
and is not limited to deep learning
\cite{yang1999information,zhang2002wavelet,tsybakov2008}.
In this paper, we assume the noise follows a Gaussian distribution.

In this scenario,
a neural network (architecture) is treated as a set of functions $\F\subset\{f:[0, 1]^d\to\R\}$.
Other estimation methods such as kernel ridge regression and wavelet threshold estimators
are also regarded as such mappings
\cite{bishop2006,donoho1994ideal}.
In this paper, we evaluate the performance of estimators by the expected mean squared error
$\mathrm{E}\left[\|\hat{f}-f^\circ\|_{L^2}^2\right]$ 
(we call this quantity the ``estimation error'' for simplicity) dependent on $n$
following convention \cite{wang2014adaptive,schmidt2017nonparametric,suzuki2018adaptivity},
where the expectation is taken with respect to the training data.
Usually, the $L^2(\mathrm{P}_X)$ norm (where $\mathrm{P}_X$ is the distribution of $X_i$) has been used in existing studies
instead of the Lebesgue $L^2$ norm,
but in evaluation the upper- (and lower-) boundedness of the density is typically assumed,
so for simplicity, we treat $X_i$ as uniformly distributed.

If we fix the set of true functions $\F^\circ$
(called a hypothesis space),
\[
	\sup_{f^\circ\in\F^\circ}\mathrm{E}\left[
		\| \hat{f}-f^\circ \|_{L^2}^2
	\right]
\]
is the worst-case performance of the estimator $(X_i, Y_i)_{i=1}^n\mapsto\hat{f}$.
We are interested in its asymptotic convergence rate
with respect to $n$, the sample size.
The minimax rate is determined by the convergence rate of
\[
	\inf_{(X_i, Y_i)\mapsto\hat{f}}\sup_{f^\circ\in\F^\circ}
	\mathrm{E}\left[
		\| \hat{f}-f^\circ \|_{L^2}^2
	\right],
\]
where $\inf$ is taken over all possible estimators.
We compare this to the convergence rate of
fixed (with respect to $n$) sequences of estimators
determined by some learning procedure such as deep learning
to evaluate how efficient the estimation method is.

As a competitor of deep learning,
a class of ``linear estimators'' is considered.
Here, we say an estimator is linear if it depends linearly on the outputs
$Y_i$; it is expressed as
\[
	\hat{f}(x)=\sum_{i=1}^nY_i\phi_i(x;X_1,\ldots,X_n).
\]
This estimator class includes several practical estimators such as
kernel ridge regression and the Nadaraya--Watson estimator.
The minimax rate in the class of linear estimators can be slower under some settings;
e.g.,
\begin{align*}
	\inf_{(X_i, Y_i)\mapsto\hat{f}}\sup_{f^\circ\in\F^\circ}
	\mathrm{E}\left[
		\| \hat{f}-f^\circ \|_{L^2}^2
	\right]\le
	n^{-\gamma}
	\inf_{\hat{f}:\text{linear}}\sup_{f^\circ\in\F^\circ}
	\mathrm{E}\left[
		\| \hat{f}-f^\circ \|_{L^2}^2
	\right]
\end{align*}
holds for some $\gamma>0$
(for most cases, we consider only the polynomial order).
Such situations were reported earlier by several authors  \cite{korostelev1993minimax,donoho1998minimax,zhang2002wavelet}.
In terms of deep learning analysis,
a comparison of deep learning with linear methods has been performed by \acite{imaizumi2018deep}. 
The present paper also shows the suboptimality of linear methods (Table \ref{table:1})
for sparse function classes, which we define later.
\begin{table*}[h]
	\centering
	\caption{Estimated error bound: deep learning vs. linear methods}
	\label{table:1}
	\begin{tabular}{|c||c|c|c|}
	\hline
		Target class & $B_{p, q}^s([0,1])\ (p<2)$
		& $\mathcal{I}^0_\Phi$ & $\mathcal{K}^p_\Psi\ (p<1)$
		\rule[-2mm]{0mm}{7mm}\\
		\hhline{|=#=|=|=|}
		Deep learning
		& $\tilde{\mathrm{O}}(n^{-\frac{2s}{2s+1}})$ &
		\begin{tabular}{c}
			$\tilde{\mathrm{O}}(n^{-1})$ \rule[-2mm]{0mm}{7mm} \\
			(Sec. \ref{revsec1})
		\end{tabular}
		& \begin{tabular}{c}
			$\tilde{\mathrm{O}}(n^{-\frac{2\alpha}{2\alpha+1}})\ 
			(\alpha=\frac1p-\frac12)$  \rule[-2mm]{0mm}{7mm} \\
			(Sec. \ref{revsec2})
		\end{tabular}
		\\\hline
		Reference & \acite{suzuki2018adaptivity} & 
		\multicolumn{2}{|c|}{This work} \rule[0mm]{0mm}{4mm}\\
		\hhline{|=#=|==|}
		Linear methods & 
		$\Omega(
			n^{-\frac{2\gamma}{2\gamma+1}
		})$ $(\gamma=s+\frac12-\frac1p)$ 
		& \multicolumn{2}{|c|}{
		\begin{tabular}{c}
		$\Omega(n^{-\frac12})$\rule[-2mm]{0mm}{7mm} \\
		(Sec. \ref{sec:3.2}, \ref{revsec1}, \ref{revsec2})
		\end{tabular}}
		\\\hline
		Reference & 
		\begin{tabular}{c}
			\acite{donoho1998minimax}\\
			\acite{zhang2002wavelet}
		\end{tabular} & 
		\multicolumn{2}{|c|}{This work} \rule[0mm]{0mm}{6mm}\\\hline
	\end{tabular}
	\\
	\vspace{2mm}
	{\small
	{\sf Note.}
	More precisely,
	the actual target function classes are the unit balls of the classes as written.
	$\tilde{\mathrm{O}}$ means $\mathrm{O}$ up to poly-log factors, and
	for each class shown in this table, deep learning attains the minimax-optimal rate
	in the sense of $\tilde{\mathrm{O}}$.
	$B_{p, q}^s$ denotes the Besov space with parameters $(s, p, q)$,
	and the parameters are additionally required to satisfy
	$s>1/p$ and $p, q\ge1$ or else $s=p=q=1$.}
\end{table*}

Our main contribution here is that we find a quite simple and natural
function class $\mathcal{I}_\Phi^0$
for which deep learning attains nearly the optimal rate, whereas linear methods are not able to converge faster than
the suboptimal rate $\mathrm{O}(n^{-1/2})$.
In the next subsection, we explain
how to treat and analyze deep learning in the context of statistical learning theory.


\subsection{Related work on estimation of deep neural networks}

Deep neural networks have a structure of alternating
linear (or affine) transformations and nonlinear transformations;
i.e., in one layer $x$ is transformed to $\rho(Wx-v)$,
where $W$ is a matrix and $v$ is a vector and $\rho$ is a nonlinear function
called an activation function.
It is known that the repeated operation of this transformation
gives a nice approximation of a wide class of nonlinear functions.

Traditionally, sigmoidal functions
have been commonly used as activation functions:
\[
	\sigma:\R\to\R,\quad \text{with}\quad
	\lim_{t\to\infty}\sigma(t)=1,\ \lim_{t\to-\infty}\sigma(t)=0.
\]
It is known that the set of functions realized by
shallow networks with continuous sigmoidal activation
is dense in any $L^p$ space unless the number of parameters is not limited
\cite{cybenko1989approximation}.
However, a similar result has also been shown for non-sigmoidal activation cases
\cite{sonoda2017neural}.
In particular, the Rectified Linear Unit (ReLU) activation function $\rho(x)=\max\{x, 0\}$
has shown practical performance \cite{glorot2011deep} and is now widely used.

Basically, deep learning trains a network by minimizing
the empirical risk with some regularization:
\[
	\text{minimize}\  \frac1n\sum_{i=1}^n(f(X_i)-Y_i)^2+\lambda(f)\quad
	\text{subject to}\ f\in\F,
\]
where $\lambda(f)$ is the regularization term and
$\F$ is the set of functions
that are realizations of a specific neural network architecture.
This optimization is usually carried out by Stochastic Gradient Descent (SGD)
or a variant of it, and the output is not necessarily the global minimum
\cite{goodfellow2016deep}.
In the present paper, however, we do not treat this optimization aspect, and we assume an ideal optimization.

The number of parameters in deep learning tends to be much larger than the sample size,
and hence without any regularization, deep models can overfit the training sample.
To overcome this issue,
existing studies have utilized sparse regularization
to obtain networks with a small number of nonzero parameters.
This enables us to 
obtain a tight estimate of the error bounds
using the result of approximation error analysis
(see Section \ref{sec:2.2}).

\acite{yarotsky2017error} reported the effectiveness of ReLU activation
in terms of approximation ability,
and the result has been exploited in estimation theory for deep learning
\cite{schmidt2017nonparametric,suzuki2018adaptivity}.
Their target function classes are H\"{o}lder space $C^s$ and Besov space $B_{p, q}^s$,
which are compatible with functional analysis or
the theory of differential equations.
\acite{schmidt2017nonparametric} also pointed out that
deep learning is superior to linear estimators
when the target function is of the form $f^\circ(x)=g(w^\top x)$,
with $g$ having some H\"{o}lder smoothness.
In addition, \acite{imaizumi2018deep} treated estimation theory for piecewise smooth functions
using the approximation theory described in \acite{petersen2018optimal}.
This paper investigates new target classes
to demonstrate the effectiveness of deep learning with ReLU activation.


\subsection{Contribution of this paper}\label{sec:contribution}

The situations where deep learning shows speriority to linear estimators
have been studied with particular theories:
\begin{itemize}
	\item
		$f^\circ(x)=g(w^\top x)$ with $g$ being smooth \cite{schmidt2017nonparametric},
	\item
		piece-wise smooth functions \cite{imaizumi2018deep},
	\item
		Besov spaces in a certain rage of parameters \cite{suzuki2018adaptivity}.
\end{itemize}
These situations can be understood mathematically through
the non-convexity of models.
To highlight this property,
we introduce a sparse function class (and therefore non-convex)
and investigate the generalization ability of deep learning and linear estimators
over the space to see how sparsity (non-convexity) affects the estimation error.
We indeed show that linear estimators have a critical disadvantage when the target class
is non-convex (Section \ref{chap:3}).
In contrast, deep learning is shown to have the optimality
over the sparse function classes we define (Section \ref{chap:5}).
These classes are natural in the sense we describe in the next paragraph.

\begin{table*}[h]
	\centering
	\caption{Correspondence between mathematical features and the real-world things}
	\label{table:rev1}
	\begin{tabular}{|c||c|c|}
		\hline
			Feature & Real-world counterpart & Examples \\
		\hhline{|=#=|=|}
			nonparametric regression & real-world estimation problems
			& \begin{tabular}{c}
				image denoising, \\
				classification problem,
				style transfer
			\end{tabular}\\
		\hline
			basis function & localized pattern
			& \begin{tabular}{c}
				pronunciation of vowels,\\
				painting style,
				fractal structure in nature
			\end{tabular}\\
		\hline
			\begin{tabular}{c}
				sparse combination \\
				of basis functions
			\end{tabular}
			& real data
			& \begin{tabular}{c}
				speech of one person, \\
				a painting,
				a noisy picture 
			\end{tabular}\\
		\hline
			sparsity
			& \begin{tabular}{c}
				low dimensionality \\
				of actual information
			\end{tabular}
			& \begin{tabular}{c}
				natural language, \\
				handwriting alphabets
			\end{tabular}\\
		\hline
	\end{tabular}
\end{table*}

The major difference between existing studies and this work is that
we assume an explicit sparsity of target classes,
which are defined parametrically.
This kind of scenario seems to occur in practice; for example, 
speech data for a specific person
are supposed to be a sparse linear combination of the person's pronunciation of each letter,
and paintings by a specific painter
may be regarded as combinations of patterns (see Table \ref{table:rev1}).
Indeed,
when we carry out the wavelet expansion to a natural image,
its relatively few large coefficients can well reconstruct the original image,
as is exploited in the field of compressive sensing \cite{candes2008}.

\begin{figure}[h]
	\hspace{-1.25cm}
	\begin{tikzpicture}[every node/.style={rectangle,rounded corners,text centered, text width=10cm}]
    	\node[draw](basis){
			basis function $\psi$\\
			\begin{minipage}{0.45\hsize}
				\includegraphics[width=4cm]{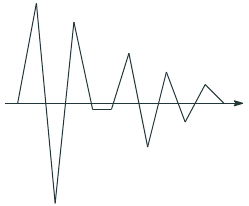}
			\end{minipage}
			\begin{minipage}{0.45\hsize}
				\vspace{-0.9mm}
				\includegraphics[width=4cm]{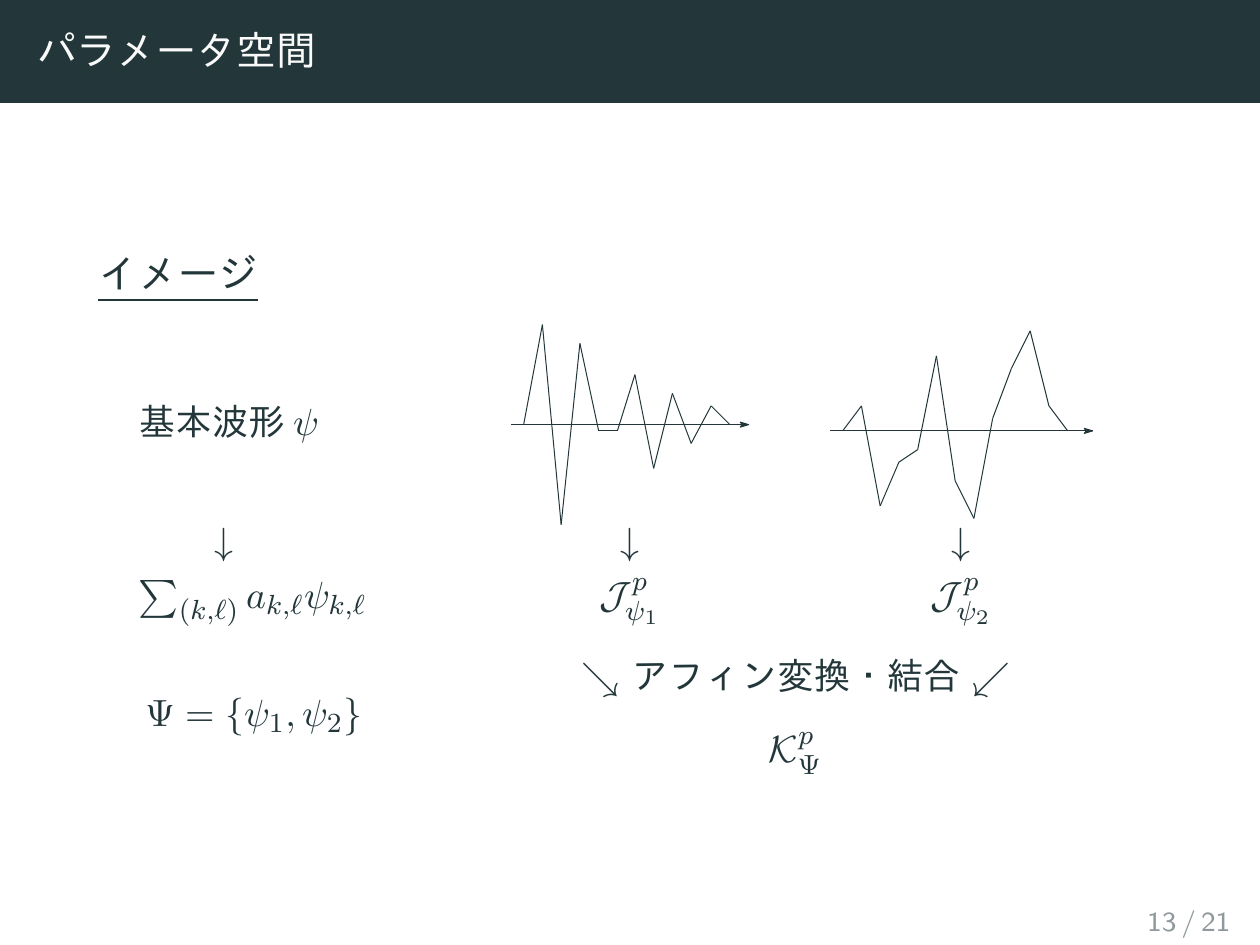}
			\end{minipage}\\
			\begin{minipage}{0.45\hsize}
				\centering$\psi_1$
			\end{minipage}
			\begin{minipage}{0.45\hsize}
				\centering$\psi_2$
			\end{minipage}
    	};
    	\node[below=2.5cm of $(basis.south east)!0.733!(basis.south west)$](p1){
    		\Large$\mathcal{J}^p_{\psi_1}$
    	};
    	\draw[->]($(basis.south east)!0.733!(basis.south west)$])
    	--node[draw, fill=white, text width=3.5cm, sharp corners]{infinite \\ sparse combination}(p1);
    	\node[below=2.5cm of $(basis.south east)!0.267!(basis.south west)$](p2){
    		\Large$\mathcal{J}^p_{\psi_2}$
    	};
    	\draw[->]($(basis.south east)!0.267!(basis.south west)$])
    	--node[draw, fill=white, text width=3.5cm, sharp corners]{infinite \\ sparse combination}(p2);
    	\node[below=5.85cm of $(basis.south east)!0.05!(basis.south west)$](kp){
    		\Large$\mathcal{K}^p_\Psi$
    	};
    	\draw[->](p1)--(kp);
    	\draw[->](p2)--(kp);
    	\node[below=4.33cm of $(basis.south east)!0.2!(basis.south west)$, draw, fill=white, text width=3.5cm, sharp corners]{finite combination};
    	\node[below=5.85cm of $(basis.south east)!0.95!(basis.south west)$](i0){
    		\Large$\mathcal{I}^0_{\Psi}$
    	};
    	\draw[->]($(basis.south east)!0.95!(basis.south west)$)--(i0);
    	\node[below=4.33cm of $(basis.south east)!0.8!(basis.south west)$, draw, fill=white, text width=3.5cm, sharp corners](i0text){finite combination};
	\end{tikzpicture}
	\caption{Interpretation of function classes}\label{fig:rev2}
\end{figure}
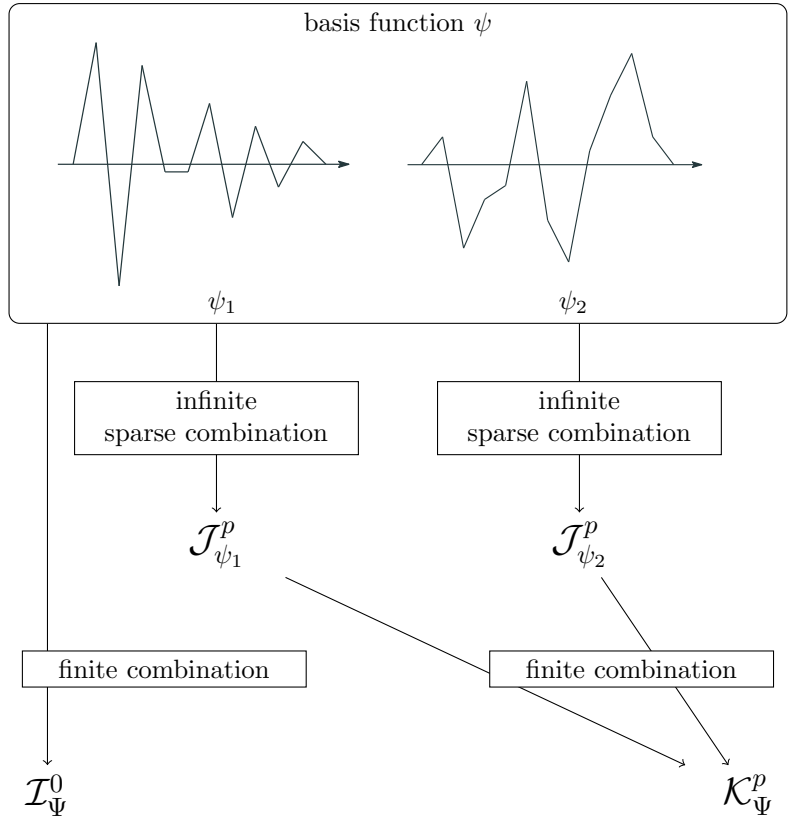

We reflect this property to define
new function classes (Section \ref{chap:4}).
In particular, we introduce a function class $\mathcal{I}_\Phi^0$ and $\mathcal{K}_\Psi^p$
(and $\mathcal{J}_\psi^p$ as well)
with a parameter
$p>0$ controlling the ``sparsity'' of the function class.
Figure \ref{fig:rev2} gives an interpretation of the motivations and definitions of these function classes.
More precisely, $p$ controls the sparsity of coefficients of
``infinite sparse combination'' appearing in Figure \ref{fig:rev2}
through a criterion called ``weak $\ell^p$ norm'' 
(see Definition \ref{def:weak_el_p} for details).
Here, ``combination" roughly means linear combination,
but precisely includes affine transform in the input; i.e.,
the combination of $f_1,\ldots, f_m$ can be generally expressed as $f=\sum_{i=1}^m c_if_i(A_i\cdot-b_i)$.

On the basis of our new function classes,
we show that deep learning is superior to linear methods and actually
attains nearly the minimax-optimal rate over each target class
(Table \ref{table:1}).
As an extreme case (corresponding to the sparsity level $p=0$),
we treat the class of piecewise constant functions,
for which the convergence rate of the linear estimators is $\Omega(n^{-1/2})$,
whereas deep learning attains the near-minimax rate $\tilde{\mathrm{O}}(n^{-1})$.
This quite simply demonstrates the scenario described in \acite{imaizumi2018deep}.
For $0<p<1$, we also show that deep learning attains the nearly minimax-optimal rate
$\tilde{\mathrm{O}}(n^{-\frac{2\alpha}{2\alpha+1}})$
in estimating a function in $\mathcal{K}^p_\Psi$, where $\alpha=1/p-1/2$.
Here, we have $\frac{2\alpha}{2\alpha+1}>\frac12$, 
and the difference between deep and linear becomes larger as $p$ becomes smaller
(i.e., as the sparsity becomes stronger).
$K_\Psi^p$ has another parameter, $\beta$,
which controls the rate of decay of coefficients of the function class.
Surprisingly, we even find that the minimax rate of linear estimators can become arbitrarily slow
under the same sparsity $p$ (and the same order of covering entropy)
if the value of $\beta$ is varied.
Although we do not yet have the upper bound
for the convergence rate of deep learning over the range of parameter values producing this situation (see Theorem \ref{linear_not_efficient}
and Remark \ref{rem:nolinear}),
this indicates that the difference between deep learning
and linear estimators could be arbitrary large. 
These differences essentially arise from the non-convexity of the model.
That is, as the non-convexity of the model becomes stronger, the difference becomes larger.

In addition, we see that deep learning takes advantage of wavelet expansions with sparsity
because a neural network can efficiently approximate functions of the form
$\sum_ic_if(A_i \cdot-b_i)$
if its subnetwork can approximate the ``basis'' function $f$ precisely
as is also mentioned in \acite{bolcskei2017optimal}.
From this perspective,
we see that parameter sharing, mentioned in Section \ref{sec:p_share},
is also effective.
It can also be said that
this paper expands the approximation theory argued in \acite{bolcskei2017optimal}
to estimation theory over sparse parameter spaces.

Thus, the contribution of this paper is summarized as follows:
\begin{itemize}
	\item
		To deal with sparsity in machine learning,
		we define function classes $\mathcal{I}_\Phi^0$ and $\mathcal{K}_\Psi^p$
		with a parameter $p$ controlling the sparsity.
		We also consider the nonparametric regression problem on
		these target classes
		and derive the minimax lower bounds of estimation error.
	\item
		We consider linear estimators, which are a competitor of deep learning,
		investigating them by evaluating their estimation error
		over sparse target classes.
		We show that linear estimators can only attain suboptimal rates on sparse and
		non-convex models
		and even become arbitrarily slow under the same sparsity
		with other parameters varying.
		This also gives a unified understanding of existing studies which describe situations
		where deep learning is superior to linear methods.
	\item
		To demonstrate the learning ability of the deep ReLU network on sparse spaces,
		we construct sparse neural networks that nearly attain minimax-optimal rates.
		It is also shown that parameter sharing in the construction of neural networks is effective
		on sparse target classes.
\end{itemize}

We give a brief overview of each section in the following.

In Section \ref{chap:2},
we introduce general methods used in statistical learning theory,
presenting our own proofs or arguments to the maximum extent possible.
Section \ref{sec:2.2} presents an information-theoretic way to obtain a lower bound
for the minimax rate.
Also, the method for evaluating an estimation error
by using an approximation error is given.
Evaluations of linear minimax rates are given in Section \ref{chap:3}.
We prove that linear estimators cannot distinguish between a function class and its convex hull,
and as a consequence linear minimax rates can be rather slower than
ordinal minimax rates.
Section \ref{chap:4} provides the definitions of our own target function classes.
The $\ell^0$ norm and the $w\ell^p$ quasi-norm of coefficients in linear combinations are introduced
as indicators of sparsity.
The minimax lower bounds for the defined classes are also given
(which are revealed to be optimal up to log factors in the section that follows).
In Section \ref{chap:5},
we show that deep learning attains the nearly minimax rate for defined function classes.
In addition, we propose that parameter sharing can be a means of reducing complexities
in regularized networks.
Finally, Section \ref{chap:6} provides a summary and presents future directions for this work.

\subsection{Notation}
We use the following notation throughout the paper.
\begin{itemize}
	\item
		$\|\cdot\|_\infty$ and $\|\cdot\|_0$ are defined as
		\[
			\|v\|_\infty:=\max_{1\le i\le m}|v_i|,\quad
			\|v\|_0:=\left|\{
				1\le i\le m \mid
				v_i\ne0
			\}\right|
		\]
		for a vector $v=(v_1,\ldots,v_m)^\top\in\R^m$.
		They are defined similarly for real matrices.
	\item
		As a natural extension of $\|\cdot\|_0$,
		$\|a\|_{\ell^0}$ denotes the number of nonzero terms in the sequence
		$a=(a_i)_{i=1}^\infty$.
	\item
		For $p>0$ and a real sequence $a=(a_i)_{i=1}^\infty$,
		the $\ell^p$ norm of $a$ is defined as
		\[
			\|a\|_{\ell^p}:=\left(\sum_{i=1}^\infty |a_i|^p\right)^{1/p},
		\]
		and $\ell^p$ denotes the set of all real sequences with a finite $\ell^p$ norm.
\end{itemize}

\section{General theories in statistical estimation}\label{chap:2}

\subsection{General settings and notation}

Let us consider the following regression model.
We observe i.i.d. random variables $(X_i, Y_i)$ generated by
\begin{align}
	Y_i=f^\circ(X_i)+\xi_i,\qquad i=1,2,\ldots,n.
	\label{gauss}
\end{align}
Here, each $\xi_i$ is an observation noise independent of other variables.
In this paper, we use settings such that
	each $X_i$ is $d$-dimensional and uniformly distributed on $[0, 1]^d$,
	each $Y_i$ is one-dimensional, and $\xi_i$'s are i.i.d. centered Gaussian variables with variance $\sigma^2$ ($\sigma>0$).
For simplicity,
we sometimes use the notation $X^n:=(X_1, \ldots, X_n)$,
$Y^n:=(Y_1,\ldots, Y_n)$, and $Z^n:=(X_i, Y_i)_{i=1}^n$.


\begin{rem}
	In the following,
	we often write only $\hat{f}$ to indicate an estimator
	where we should write $(X_i, Y_i)^n\mapsto\hat{f}$
	(this mapping is supposed to be measurable).
	For example, $\inf_{(X_i, Y_i)_{i=1}^n\mapsto\hat{f}\in\F}$ is simply denoted by
	$\inf_{\hat{f}\in\F}$.
	In addition, for the case $\F=L^2([0, 1]^d)$, we omit $\F$ and simply write $\inf_{\hat{f}}$.
\end{rem}

To evaluate the quality of estimators,
we need to adopt some evaluation criteria.
For a fixed $f^\circ$ and a function $f\in L^2([0,1]^d)$,
we have
\begin{align*}
	\mathrm{E}[(f(X)-Y)^2]
	&=\mathrm{E}[(f(X)-f^\circ(X))^2]-2\mathrm{E}[\xi(f(X)-f^\circ(X))]+\mathrm{E}[\xi^2]\\
	&=\mathrm{E}[(f(X)-f^\circ(X))^2]+\sigma^2\\
	&=\|f-f^\circ\|_{L^2}^2+\sigma^2.
\end{align*}
This implies that the magnitude of the expected error $\mathrm{E}[(f(X_i)-Y_i)^2]$ depends only 
on that of the $L^2$ distance $\|f-f^\circ\|_{L^2}^2$.
This leads to the following definition for a performance criterion.

\begin{dfn}
The {\it$L^2$ risk} for an estimator $\hat{f}$ is defined as
\[
	R(\hat{f}, f^\circ)
	:=\mathrm{E}\left[\|\hat{f}-f^\circ\|_{L^2}^2\right].
\]
For a model $\F^\circ\subset L^2([0, 1]^d)$,
the {\it minimax $L^2$ risk} over $\F^\circ$ is defined as
\[
	\inf_{\hat{f}}\sup_{f^\circ\in\F^\circ}R(\hat{f}, f^\circ)
	=\inf_{\hat{f}}\sup_{f^\circ\in\F^\circ}\mathrm{E}\left[\|\hat{f}-f^\circ\|_{L^2}^2\right],
\]
where $\hat{f}$ runs over all estimators (measurable functions).
\end{dfn}

We evaluate the quality of an estimator $\hat{f}$ by this $L^2$ risk
and compare it with the minimax-optimal risk.

\begin{rem}
	We omit $n$ from the notation because it is treated as a constant when we consider
	a single regression problem.
	However, as $n$ goes to $\infty$,
	the minimax risk converges to $0$,
	and in this paper we are interested in the convergence rate of the minimax risk.
\end{rem}


\subsection{Relationships between complexity and minimax risk}\label{sec:2.2}

In this section, we introduce a key procedure to evaluate the minimax risk.
To do so, we define complexity measures called {\it $\ve$-entropy}
($\ve$ may be taken place by any positive real number),
which formally represent complexities of (totally bounded) metric spaces.
This kind of complexity of $\F^\circ$ profoundly affects the convergence rate of the minimax risk
\cite{yang1999information}.

\begin{dfn}\cite{van1996weak,yang1999information}
	For a metric space $(S, d)$ and $\ve>0$,
	\begin{itemize}
		\item a finite subset $T$ is called {\it $\ve$-packing}
		if $d(x, y)>\ve$ holds for any $x, y\in T$ with $x\ne y$,
		and the logarithm of the maximum cardinality of an $\ve$-packing subset
		is called the {\it packing $\ve$-entropy} and is denoted by $M_{(S, d)}(\ve)$;
		\item a finite set $U\subset \overline{S}$
		is called {\it $\ve$-covering}
		if for any $x\in S$ there exists $y\in U$ such that $d(x, y)\le\ve$,
		and the logarithm of the minimum cardinality of an $\ve$-covering set
		is called the {\it covering $\ve$-entropy} and is denoted by $V_{(S, d)}(\ve)$.
	\end{itemize}
	Here, $\overline{S}$ is the completion of $S$ with respect to the metric $d$.
\end{dfn}

The concept of $\ve$-entropy is useful to obtain a lower bound of the minimax risk
of some function class $\F^\circ$.
Let $\F^\circ\subset L^2([0, 1]^d)$ be the class of true functions,
equipped with the $L^2$ metric.
For simplicity,
let $V(\ve)=V_{(\F^\circ, \|\cdot\|_{L^2})}(\ve)$
and $M(\ve)=M_{(\F^\circ, \|\cdot\|_{L^2})}(\ve)$.
Then, the following theorem holds.

\begin{thm}{\rm\cite[Theorem 1]{yang1999information}}\label{yang_barron}
In the Gaussian regression model, 
suppose there exist $\delta, \ve>0$ such that
\[
	V(\ve)\le\frac{n\ve^2}{2\sigma^2},\quad
	M(\delta)\ge \frac{2n\ve^2}{\sigma^2}+2\log2.
\]
Then we have
\begin{align*}
	\inf_{\hat{f}}\sup_{f\in\F^\circ}\mathrm{P}_f
	\left(\|{\hat{f}-f}\|_{L^2}\ge\frac{\delta}2\right)
	\ge \frac12,\quad
	\inf_{\hat{f}}\sup_{f\in\F^\circ}
	\mathrm{E}_f\left[\|{\hat{f}-f}\|_{L^2}^2\right]
	\ge\frac{\delta^2}8,
\end{align*}
where $\mathrm{P}_f$ is the probability law with $f^\circ=f$,
and $\mathrm{E}_f$ is the expectation determined by $\mathrm{P}_f$.
\end{thm}

The above theorem states the relationship between
the complexity of {\it the function class} and the {\it lower} bound of the minimax risk.
On the other hand, the following theorem is relates the complexity of {\it the estimator} and the {\it upper} bound
of the generalization error (and the minimax risk, at the same time).
The following is also useful for evaluating the convergence rate
of the empirical risk minimizer of some explicit model, such as neural networks.

\begin{thm}{\rm\cite[Lemma 4]{schmidt2017nonparametric}}\label{gen_eval}
	In the Gaussian regression model (\ref{gauss}),
	let $\hat{f}$ be the empirical risk minimizer, taking values in $\F\subset L^2([0, 1]^d)$.
	Suppose every element $f\in\F$ satisfies $\|f\|_{L^\infty}\le F$ for some fixed $F>0$.
	Then, for an arbitrary $\delta>0$,
	if $V_{(\F, \|\cdot\|_{L^\infty})}(\delta)\ge 1$, then
	\begin{align*}
		R(\hat{f}, f^\circ)
		\le 4\inf_{f\in\F}\|f-f^\circ\|_{L^2}^2
		+C\biggr(\frac{(F^2+\sigma^2)V_{(\F, \|\cdot\|_{L^\infty})}(\delta)}n
		+(F+\sigma)\delta\biggl)
	\end{align*}
	holds, where $C>0$ is an absolute constant.
\end{thm}
This sort of evaluation has been obtained earlier in
\cite{gyorfi2006distribution,koltchinskii2006local,gine2006concentration}
and the proof for this is essentially the same as earlier ones.
For completeness, however, we give a proof for this assertion
in the appendix (Section \ref{ap-2}).\footnote{We noticed some technical flaws
in an earlier version of the proof of \acite{schmidt2017nonparametric},
so we include the proof in the appendix for completeness.}

\section{Suboptimality of linear estimators}\label{chap:3}

We consider linear estimators as a competitor to deep learning,
and in this section, we characterize their suboptimality by the convexity of the target model.
Linear estimators, represented by kernel methods,
are classically applied to regression problems.
Indeed,
some linear estimators have minimax optimality
over smooth function classes
such as
H\"{o}lder classes and Besov classes with some constraint on their parameters
\citep[with fixed design:][]{donoho1998minimax,tsybakov2008}.
However,
as has been pointed out in the literature \cite{korostelev1993minimax,imaizumi2018deep},
linear estimators can attain only suboptimal rates with function classes having discontinuity.
We here show that
the suboptimality of linear estimators arises
even with a quite simple target class.
Our first contribution is to point out that
the concept of the convex hull gives the same explanation to such suboptimality
for several target classes,
and based on that argument, we then show that linear estimators perform suboptimally
even on a quite simple target class.

\subsection{Linear estimators and its minimax risk}

\begin{dfn}
	The estimation scheme $(X_i, Y_i)_{i=1}^n\mapsto \hat{f}$ is called {\it linear} if $\hat{f}$
	has the form
	\[
		\hat{f}(x)=\sum_{i=1}^nY_i\varphi_i(x; X^n),
	\]
	where we suppose $\mathrm{E}\left[\|\phi_i(\cdot;X^n)\|_{L^2}^2\right]<\infty$.
	Also, we call an estimator $\hat{f}$ {\it affine} if $\hat{f}$ has the form
	\[
		\hat{f}(x)=\hat{f}_L(x)+\phi(x;X^n),
	\]
	where $\phi$ has the same condition as $\phi_i$, and $\hat{f}_L$ is a linear estimator. 
\end{dfn}

\begin{rem}
	The condition $\mathrm{E}\left[\|\phi_i(\cdot;X^n)\|_{L^2}^2\right]<\infty$
	may be replaced by a weaker version.
	This actually assures that
	\begin{itemize}
		\item
			$\phi_i(\cdot;X^n)\in L^2([0, 1]^d)$ holds almost surely;
		\item
			$\mathrm{E}\left[\phi_i(x;X^n)^2\right]<\infty$ holds almost everywhere.
	\end{itemize}
	The latter condition is only needed in the justification of (\ref{eq:zhang1}).
\end{rem}

A linear estimator is of course an affine estimator as well.
Linear or affine estimators are classically used often; they include
linear (ridge) regression, the Nadaraya--Watson estimator,
and kernel ridge regression
\cite{tsybakov2008,bishop2006,friedman2001}.
For example,
the estimator given by kernel ridge regression
can be explicitly written as
\[
	\hat{f}(x):=(k(x, X_1), \ldots, k(x, X_n))(K+\lambda I_n)^{-1} (Y_1,\ldots,Y_n)^\top,
\]
where $\lambda$ is a positive constant, $k:[0, 1]^d\times [0,1]^d\to\R$
is a positive semi-definite kernel, and
the matrix $K\in\R^{n\times n}$ is defined as
$K:=(k(X_i, X_j))_{i,j}$.
We can see that the difference in performance between deep learning and linear estimators
becomes large in a non-convex model,
which can be explained by the following theorem.
(This theorem can also be seen as a generalization of \acite[Theorem 5]{cai2004minimax}.)

Let $\conv(\F^\circ)$ denote the convex hull of $\F^\circ$; i.e.,
	\[
		\conv(\F^\circ):=\left\{
			\sum_{i=1}^kt_if_i
			\,\middle|\,
				t_1,\ldots, t_k\ge0,\ \sum_{i=1}^kt_i=1,\ 
				f_1,\ldots,f_k\in\F^\circ,\ k\ge1
		\right\}.
	\]
Notice that the $\conv(\F^\circ)$ is larger than the original set $\F^\circ$.
Let $\overline{\conv}(\F^\circ)$ be the closure of $\conv(\F^\circ)$
with respect to the $L^2$ metric (caller closed convex hull).
The following assertion holds.

\begin{thm}\label{thm:l2}
	For affine methods, the minimax risk over $\F^\circ$ coincides with the minimax risk
	over $\overline{\conv}(\F^\circ)$; i.e., the following equality holds:
	\[
		\inf_{\hat{f}:\rm{affine}}\sup_{f^\circ\in\F^\circ}R(\hat{f}, f^\circ)
		=\inf_{\hat{f}:\rm{affine}}\sup_{f^\circ\in\overline{\conv}(\F^\circ)}R(\hat{f}, f^\circ).
	\]
\end{thm}

\begin{proof}
	We only prove the assertion
	\begin{align}
		\inf_{\hat{f}:\rm{affine}}\sup_{f^\circ\in\F^\circ}R(\hat{f}, f^\circ)
		=\inf_{\hat{f}:\rm{affine}}\sup_{f^\circ\in\conv(\F^\circ)}R(\hat{f}, f^\circ).
		\label{l_add}
	\end{align}
	For the case of $\overline{\conv}(\F^\circ)$, see the appendix (Section \ref{convproof}).
	
	Fix an estimator $\hat{f}$ and let
	\begin{align}
		\hat{f}(x)=\phi(x;X^n)+\sum_{i=1}^nY_i\phi_i(x;X^n).
		\label{eq:l0}
	\end{align}
	For $f^\circ, g^\circ\in\F^\circ$ and $t\in(0, 1)$, let $h^\circ:=tf^\circ+(1-t)g^\circ$. Then
	\begin{align}
		R(\hat{f}, h^\circ)
		=\mathrm{E}\left[
			\int_{[0,1]^d}\left(\hat{f}(x)-h^\circ(x)\right)^2\dd x
		\right]
		=\int_{[0,1]^d}\mathrm{E}\left[\left(\hat{f}(x)-h^\circ(x)\right)^2\right]\dd x
		\label{eq:l1}
	\end{align}
	holds by Fubini's theorem (the integrated value is nonnegative).
	By the convexity of the square, we have
	\begin{align}
		\left(\hat{f}(x)-h^\circ(x)\right)^2
		&=\left(\phi(x;X^n)+\sum_{i=1}^nY_i\phi_i(x;X^n)-h^\circ(x)\right)^2\nonumber\\
		&=\Biggl(\phi(x;X^n)+\sum_{i=1}^n\xi_i\phi_i(x;X^n)
		+
		\sum_{i=1}^nh^\circ(X_i)\phi_i(x;X^n)-h^\circ(x)\Biggr)^2\label{eq:l2}\\
		&\le t\Biggl(\phi(x;X^n)+\sum_{i=1}^n\xi_i\phi_i(x;X^n)+
		\sum_{i=1}^nf^\circ(X_i)\phi_i(x;X^n)-f^\circ(x)\Biggr)^2\nonumber\\
		&\quad+(1-t)\Biggl(\phi(x;X^n)+\sum_{i=1}^n\xi_i\phi_i(x;X^n)+
		\sum_{i=1}^ng^\circ(X_i)\phi_i(x;X^n)-g^\circ(x)\Biggr)^2\nonumber\\
		&=t\left(\hat{f}(x)\bigg|_{Y_i=f^\circ(X_i)+\xi_i}-f^\circ(x)\right)^2
		+(1-t)\left(\hat{f}(x)\bigg|_{Y_i=g^\circ(X_i)+\xi_i}-g^\circ(x)\right)^2.\nonumber
	\end{align}
	Here, notice that $\hat{f}$ is dependent
	on whether we choose $f^\circ$, $g^\circ$, or $h^\circ$.
	Therefore, we integrate this inequality to obtain
	\[
		R(\hat{f},h^\circ)\le tR(\hat{f}, f^\circ)+(1-t)R(\hat{f}, g^\circ).
	\]
	This means that $R(\hat{f}, \cdot)$ is a convex functional,
	and so $\text{LHS}\ge\text{RHS}$ holds in (\ref{l_add}).
	Since it is clear that $\text{LHS}\le\text{RHS}$, the equality of (\ref{l_add}) holds.
\end{proof}

\begin{rem}
	Indeed, \acite{donoho1990minimax} and \acite{donoho1998minimax} pointed out that
	the convex hull in the above assertion can be replaced by the {\it quadratic hull},
	which is generally larger than a convex hull in a similar setting (fixed design).
	However, their propositions require the assumption
	of fixed design and  {\it orthosymmetricity}
	with some wavelet expansion.
	Hence, we have explicitly noted Theorem \ref{thm:l2} under milder conditions.
\end{rem}

By this theorem,
we see that linear estimators hardly achieve the minimax rate
in a non-convex model.
This also explains the difference between deep learning and linear methods
argued in \acite{schmidt2017nonparametric},
\acite{imaizumi2018deep} and \acite{suzuki2018adaptivity}
in a unified manner.
In the following section,
we demonstrate a simple example where linear estimator are suboptimal.


\subsection{Functions of bounded total variation}\label{sec:3.2}

Let us consider a specific function class as a simple but instructive example, 
a class whose convex hull becomes larger in terms of the covering entropy.
In addition, the convex hull is dense in $\mathrm{BV}(C)$ (defined below),
over which linear estimators can only attain a suboptimal rate.

\begin{dfn}\label{def:jk}
	For $k\ge 1$ and $C>0$,
	define
	\[
		J_k(C)
		:=\left\{
		a_0+\sum_{i=1}^k a_i1_{[t_i, 1]}
		\,\middle|\,
		t_i\in(0, 1],\ 
		|a_0|\le C,\
		\sum_{i=1}^k|a_i|\le C
		\right\}
	\]
	as functions from $[0, 1]$ to $\R$ with jumps occurring at most $k$ times.
\end{dfn}

We can also understand $J_k(C)$ as the set of piece-wise constant functions.

\begin{figure}[h]
	\centering\includegraphics[width=0.7\hsize]{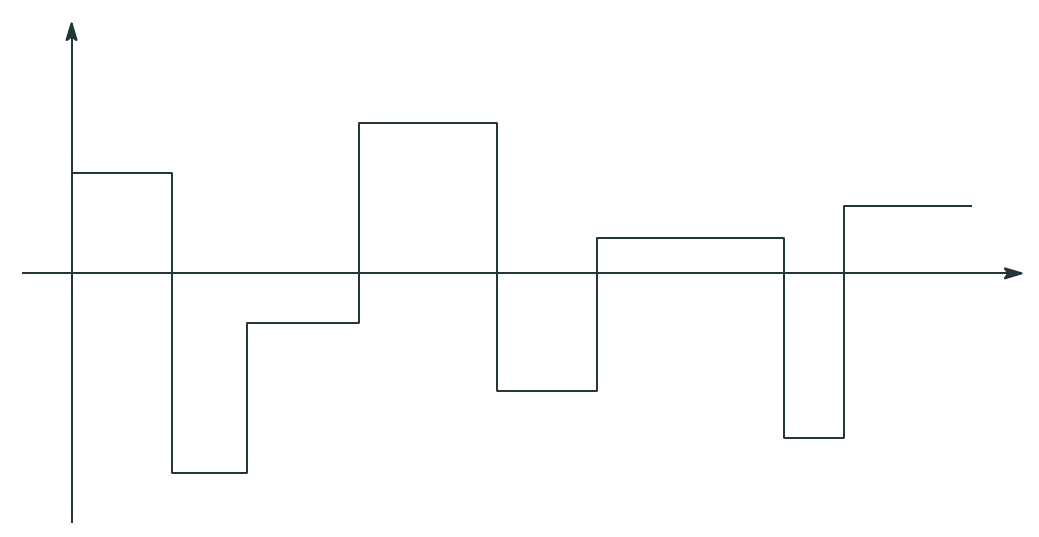}
	\caption{An example of piece-wise constant functions}
\end{figure}

We next introduce a function class well-known in the field of real analysis,
which is indeed related to $J_k$'s (Lemma \ref{lem:ll}).

\begin{dfn}
	For any real numbers $a<b$ and a function $f:[a, b]\to\R$,
	define the {\it total variation} of $f$ on $[a, b]$ as
	\[
		\mathrm{TV}^f([a, b])
		:=\sup_{
			M\ge1,\ 
			a=t_0<\cdots<t_M=b}
			\sum_{i=0}^{M-1}|f(t_{i+1})-f(t_i)|.
	\]
	Also, for $C>0$, define the set of functions with bounded total variation as
	\[
		\mathrm{BV}(C):=\left\{
			f:[0, 1]\to\R
			\,\middle|\,
			|f(0)|\le C,\ 
			\mathrm{TV}^f([0,1])\le C
		\right\}.
	\]
\end{dfn}

\begin{rem}
	The condition $|f(0)|\le C$ is needed to bound the size of the set,
	and it may be replaced by other similar bounding conditions such as
	$\sup_{t\in[0,1]}|f(t)|\le C$ or $\int_0^1|f(t)| \dd t$
	$\le C$
	\citep[e.g.,][]{donoho1993unconditional}.
	These conditions are equivalent up to constant multiplications of $C$ (i.e.,
	$\mathrm{BV}(C)\subset \mathrm{BV}'(\alpha C)
	\subset \mathrm{BV}(\beta C)$ holds for some $\alpha, \beta>0$,
	where $\mathrm{BV}'$ is a set defined with another constraint).
	Hence, we adopt $|f(0)|\le C$ for simplicity of arguments.
\end{rem}

Then, we have the following assertion.
The proof is given in the appendix (Section \ref{rev_lem:ll})
	
\begin{lem}\label{lem:ll}
	$\overline{\conv}(J_k(C))\supseteq \mathrm{BV}(C)$ holds for each $k\ge1$ and $C>0$.
\end{lem}

From the above, it follows that linear estimators cannot distinguish $J_k(C)$ and $\mathrm{BV}(C)$
in terms of minimax convergence rates (Theorem \ref{thm:l2}).

Since it is known that the unit ball of $B_{1, 1}^1([0, 1])$ is included
in $\mathrm{BV}(C)$ for some $C>0$ \cite{peetre1976new},
the following theorem can be seen as a special case of Theorem $1$ in \acite{zhang2002wavelet}
(see also Table \ref{table:1}).

\begin{thm}\label{thm:bv}
	There exists a constant $c>0$ dependent only on $C$ such that
	\[
		\inf_{\hat{f}:\rm{linear}}\sup_{f^\circ\in \mathrm{BV}(C)}
		R(\hat{f}, f^\circ)
		\ge cn^{-1/2}
	\]
	holds.
\end{thm}

The following corollary is one of the main results in this paper.

\begin{cor}\label{cor:jk}
	For $k=1,2,\ldots$,
	there exists a constant $c>0$ dependent only on $C$ such that
	\[
		\inf_{\hat{f}:\rm{linear}}\sup_{f^\circ\in J_k(C)}
		R(\hat{f}, f^\circ)
		\ge cn^{-1/2}
	\]
	holds.
\end{cor}

\begin{proof}
	The result is clear from Theorem \ref{thm:l2}, Lemma \ref{lem:ll}, and Theorem \ref{thm:bv}.
	We can of course take the same $c$ as in Theorem \ref{thm:bv}.
\end{proof}
	
\begin{rem}
	On the one hand, the minimax-optimal rate 
	of the unit ball of $B^1_{1, 1}([0,1])$ is $\tilde{\Theta}(n^{-2/3})$
	(Table \ref{table:1}),
	whereas the counterpart of $J_k(C)$ is $\tilde{\Theta}(n^{-1})$ as is attained by deep learning
	(proved later; see Corollary \ref{cor:trivial}).
	On the other hand, the fact that the unit ball of $B^1_{1, 1}([0, 1])$
	is included in $\mathrm{BV}(C)$ implies that
	the linear minimax rate of $\mathrm{BV}(C)$
	is not faster than that of $B^1_{1, 1}([0,1])$'s unit ball.
	Since $\mathrm{BV}(C)$ and $J_k(C)$ have the same linear minimax rate,
	$J_k(C)$ is a quite extreme example, even in comparison with $B^1_{1, 1}([0,1])$.
\end{rem}

\section{Sparse target function classes}\label{chap:4}

As $J_k(C)$'s given in Section \ref{sec:3.2} are too simple,
we consider sparse target function classes
which are generalizations of $J_k(C)$'s
and are also related to wavelets.
We investigate the performance of deep learning and other methods
over these sparse classes.
The minimax lower bound for each class is also given
by applying the arguments in Section \ref{sec:2.2}.
Sparsity well characterizes the spaces whose convex hulls
are much larger than the original spaces,
a property that is essential for the proofs that were given in Section \ref{chap:3}.


\subsection{The $\ell^0$-bounded affine class}

The definition of the following class
is inspired by the concept of ``affine class'' treated in \acite{bolcskei2017optimal}.

\begin{dfn}\label{dfn:p=0}
	Given a set $\Phi\subset L^2([0, 1]^d)$
	with $\|\phi\|_{L^2}=1$ for each $\phi\in\Phi$
	along with constants $n_s\in\Z_{>0}$ and $C>0$,
	we define an {\it $\ell^0$-bounded affine class $\mathcal{I}_\Phi^0$} as
	\[
		\mathcal{I}_\Phi^0(n_s, C):=
		\left\{
			\sum_{i=1}^{n_s} c_i\phi_i(A_i\cdot-b_i)
			\,\middle|\,
				|\det A_i|^{-1},\ \|A_i\|_\infty,\ 
				\|b_i\|_\infty,\ |c_i|\le C,\ 
				\phi_i\in\Phi,\ i=1,\ldots,n_s
		\right\}.
	\]
\end{dfn}

If we adopt a jump-type function $\phi\in\Phi$ such as $\phi=1_{[0, 1/2)}$,
we can see $\mathcal{I}_\Phi^0(n_s, C)$ as a generalization
of $J_k(C)$'s defined in the previous section.
The condition for $c_i$ is also regarded as $\|c\|_{\ell^0}\le n_s$,
where the $\ell^0$ norm is used as the most extreme measurement of sparsity
\cite{raskutti2011minimax,wang2014adaptive}.

Let us derive a minimax lower bound for this class.
Although the proof for this assertion can easily be given
by applying the argument appearing in \acite{tsybakov2008},
we provide it in the appendix (Section \ref{ap-3}).

\begin{thm}\label{thm:weak_lower_bound}
	There exists a constant $C_0>0$ depending only on $\sigma^2$ such that
	\[
		\inf_{\hat{f}}\sup_{f^\circ\in\mathcal{I}_\phi^0}
		R(\hat{f}, f^\circ)
		\ge\frac{C_0}n
	\]
	holds for each $n\ge 1$.
\end{thm}

\begin{rem}
	If the set $\Phi$ is simple enough; e.g., a finite set of
	piece-wise constants or piece-wise polynomials,
	then this rate can be almost attained by nonlinear estimators.
	Indeed, in such cases deep learning attains the rate $\mathrm{O}(n^{-1}(\log n)^3)$
	(Theorem \ref{thm:first_order}).
	However, with a single noncontinuous $\phi\in\Phi$,
	linear estimators become suboptimal;
	the linear minimax rate is lower-bounded by $\mathrm{\Omega}(n^{-1/2})$
	(Corollary \ref{cor:jk}).
\end{rem}


\subsection{The $w\ell^p$-bounded function classes}

The function class we treat in the previous section seems too simple
to approximate the real-world data.
Therefore, we are going to consider larger function classes with ``sparsity''.
We introduce concepts for measuring the sparsity of function classes
in order to present a simple treatment of several sparse spaces.
These concepts were introduced and discussed previously in
\acite{donoho1993unconditional}, \acite{donoho1996unconditional}, and
\acite{yang1999information}.

\begin{dfn}\label{def:weak_el_p}
	For a sequence $a=(a_i)_{i=1}^\infty\in\ell^2$,
	let each $|a|_{(i)}$ denote the $i$-th largest absolute value of terms in $a$.
	For $0<p<2$, the {\it weak $\ell^p$ norm} of $a$ is defined as
	\begin{align}
		\|a\|_{w\ell^p}:=\sup_{i\ge 1}i^{1/p}|a|_{(i)}.
		\label{eq:sp1}
	\end{align}
\end{dfn}

Here, notice that $\|\cdot\|_{w\ell^p}$ is not a norm, as $(|a|_{(i)})_{i=1}^\infty$
is a permutation of $(|a_i|)_{i=1}^\infty$.
However, we call it a ``weak $\ell^p$ norm'' following the notation used in
\acite{donoho1993unconditional} and \acite{donoho1996unconditional}. 

\begin{dfn}\label{p_tail_comp}
	Given an orthonormal set $\phi=(\phi_i)_{i=1}^\infty\subset L^2([0, 1]^d)$ and
	constants $C_1, C_2, \beta>0$ and $0<p<2$, we define a {\it sparse $\ell^p$-approximated
	set} $\mathcal{I}_\phi^p$ as
	\[
		\mathcal{I}^p_\phi(C_1, C_2, \beta):=\left\{
			\sum_{i=1}^\infty a_i\phi_i
			\,\middle|\,
			\|a\|_{w\ell^p}\le C_1,\ 
			\sum_{i=m+1}^\infty a_i^2\le C_2m^{-\beta},\ 
			m=1,2,\ldots
		\right\}.
	\]
\end{dfn}

The constraint $\sum_{i=m+1}^\infty a_i^2\le Cm^{-\beta}$
is called {\it$\beta$-minimally tail compactness} of $(a_i)_{i=1}^\infty$,
which is required to make the set compact in the $L^2$ metric.

\begin{rem}
	To represent sparsity, the $\ell^p$ norm of coefficients is also used
	\cite[see, e.g.,][]{raskutti2011minimax,wang2014adaptive}.
	Note here that $\|a\|_{\ell^p}\le C$ implies $\|a\|_{w\ell^p}\le C$.
	Indeed, $\|a\|_{\ell^p}\le C$ means that for each $i$,
	\[
		i^{1/p}|a|_{(i)}
		\le\left(\sum_{j=1}^i|a|_{(j)}^p\right)^{1/p}
		\le\left(\sum_{j=1}^\infty|a_j|^p\right)^{1/p}
		\le C.
	\]
	Thus, a weak $\ell^p$ ball contains an ordinary $\ell^p$ ball.
	In addition, consider the case in which $d=1$ and $\phi$ is an orthonormal basis
	generated by a wavelet in $C^r([0, 1])$
	with $r\in\Z_{>0}$ satisfying $r>\alpha:=1/p-1/2$.
	Then, the Besov norm $\|\cdot\|_{B^\alpha_{p, p}}$
	of a function is equivalent to the $\ell^p$-norm $\|\cdot\|_{\ell^p}$ of wavelet coefficients
	\cite[Theorem 2]{donoho1998minimax}.
	In this case, $\mathcal{I}^p_\phi$ may be just a slight expansion of existing space,
	but our main interest is the case in which $\phi$ has a discontinuity (e.g., when $\phi$
	is defined by the Haar wavelet),
	which makes things different.
	Furthermore, notice that Besov spaces with such parameters are omitted
	in Table \ref{table:1} (see the note;
	the upper bounds are given in \acite{suzuki2018adaptivity} for a wider range of parameters,
	but the range for the given lower bounds for linear estimators is limited).
\end{rem}

Hereinafter, we fix $p$, $C_1$, $C_2$, and $\beta$ and 
often write $\mathcal{I}^p_\phi(C_1, C_2, \beta)$ as $\mathcal{I}^p_\phi$
if there is no confusion;
therefore, constants appearing in the following may depend on these values.
In the following arguments,
we first derive a minimax lower bound for $\mathcal{I}_\phi^p$,
and then we introduce a broader function class that is well approximated by neural networks.

To use Theorem \ref{yang_barron}, we exploit the following lemma.
As stated in Section \ref{sec:2.2},
the covering entropy of the function class is important.
The proof of the lemma is given in the appendix
(Section \ref{lem:rev_up_low}).

\begin{lem}\label{lem:sp1}
	Let $\alpha:=1/p-1/2$, and suppose $\beta$ satisfies $\beta\le2\alpha$.
	Then there exists a constant $C_\mathrm{low}, C_\mathrm{up}>0$ such that
	\[
		C_\mathrm{low}\ve^{-1/\alpha}
		\le V_{(\mathcal{I}_{\phi}^p, \|\cdot\|_{L^2})}(\ve)
		\le C_\mathrm{up}\ve^{-1/\alpha}(1+\log(1/\ve))
	\]
	holds for each $\ve>0$.
\end{lem}

Next, we derive a nearly tight minimax lower bound for $\mathcal{I}_\phi^p$.
In this case, ``nearly'' means ``up to log factors.''

\begin{thm}\label{thm:p>0}
	There exists a constant $C=C(p, C_1, C_2)>0$ such that
	\[
		\inf_{\hat{f}}\sup_{f^\circ\in\mathcal{I}_\phi^p}
		R(\hat{f}, f^\circ)
		\ge C n^{-\frac{2\alpha}{2\alpha+1}}(\log n)^{-\frac{4\alpha^2}{2\alpha+1}}
	\]
	holds for each $n\ge 2$.
\end{thm}

\begin{proof}
	In this proof, we write the $\ve$-entropies of $\mathcal{I}_\phi^p$ simply as
	$V(\ve)$ and $M(\ve)$.
	
	First,
	let $\displaystyle\ve_n:=c\left(\frac{\log n}n\right)^{\frac{\alpha}{2\alpha+1}}$
	for some constant $c>0$.
	Then by Lemma \ref{lem:sp1}, we have
	\begin{align*}
		V(\ve_n)
		\le C_\mathrm{up}c^{-1/\alpha}\left(\frac{\log n}n\right)^{-\frac1{2\alpha+1}}
		\left(1+\frac{\alpha}{2\alpha+1}(\log n-\log\log n)\right)
		\le
		c_\mathrm{up}c^{-1/\alpha}n^{\frac1{2\alpha+1}}(\log n)^{\frac{2\alpha}{2\alpha+1}},
	\end{align*}
	where $c_\mathrm{up}>0$ is some constant independent of $\ve_n$, and we have used $n\ge 2$.
	Thus we have
	\begin{align}
		\frac{V(\ve_n)}{n\ve_n^2}
		\le
		c_\mathrm{up}c^{-2-1/\alpha}\le\frac1{2\sigma^2}
		\label{eq:sp4}
	\end{align}
	for a sufficiently large $c$.
	
	Second, notice that $M(\ve)\ge V(\ve)$ holds.
	Indeed, given a maximal $\ve$-packing of $\mathcal{I}_\phi^p$,
	the maximality implies that the set also satisfies the condition for being an $\ve$-covering.
	Now, let $\delta_n:=C'n^{-\frac{\alpha}{2\alpha+1}}(\log n)^{-\frac{2\alpha^2}{2\alpha+1}}$
	for some constant $C'>0$.
	Then we have, by Lemma \ref{lem:sp1},
	\begin{align}
		M(\delta_n)
		\ge V(\delta_n)
		\ge C_\mathrm{low}C'^{-1/\alpha}n^{\frac1{2\alpha+1}}
		(\log n)^{\frac{2\alpha}{2\alpha+1}}
		\ge C'^{-1/\alpha}c_\mathrm{low}\left(\frac{2n\ve_n^2}{\sigma^2}+2\log2\right)
		\label{eq:sp5}
	\end{align}
	for some constant $c_\mathrm{low}>0$ independent of $C'$,
	where we have used $n\ge 2$.
	
	By (\ref{eq:sp4}), (\ref{eq:sp5}), and Theorem \ref{yang_barron},
	for a sufficiently small $C'$,
	we have
	\[
		\inf_{\hat{f}}\sup_{f\in\mathcal{I}_\phi^p}\mathrm{E}
		\left[ \|f-\hat{f}\|_{L^2}^2 \right]\ge
		\frac18C'^2 n^{-\frac{2\alpha}{2\alpha+1}}(\log n)^{-\frac{4\alpha^2}{2\alpha+1}},
	\]
	and the proof is complete.
\end{proof}

\begin{rem}
	This minimax lower bound is nearly tight, especially for the wavelet case treated in the following section.
	Indeed, deep learning (if necessary, with parameter sharing)
	achieves the rate $\mathrm{O}(n^\frac{2\alpha}{2\alpha+1}(\log n)^3)$,
	for a broad range of wavelets (Theorem \ref{thm:main_p>0}, \ref{p_share_own}).
\end{rem}


\subsection{Sparsity conditions for wavelet coefficients}\label{sec:sparse_wavelet}

In this subsection, we apply the argument in the previous subsection to orthogonal wavelets.
They already have a broad application area from engineering and physics to pure and applied mathematics,
especially in signal processing, numerical analysis, and so on
\cite{ingrid1992ten}.
Because wavelets are the mathematical model of localized patterns,
they are suitable for our motivation in setting function classses
(see Section \ref{sec:contribution} and Table \ref{table:rev1}).
For simplicity, we only consider the $1$-dimensional case in this section,
and the treatment of multi-dimensional cases is deferred to the appendix (Section \ref{multi-wave}).

\begin{dfn}
	Let $\psi:[0, 1]\to\R$ be a function with $\|\psi\|_{L^2}=1$.
	For such a function,
	we define, for integers $k, \ell$,
	\[
		\psi_{k, \ell}(x):=2^{k/2}\psi(2^kx-\ell),
		\quad
		k\ge0,\ 0\le\ell<2^k,
	\]
	where $\psi$ is treated as $0$ outside $[0, 1]$.
	Also, $\psi$ is called an {\it orthogonal wavelet}
	if $\psi$ satisfies
	\[
		\int_0^1\psi_{k, \ell}(x)\psi_{k', \ell'}(x)\dd x=0
	\]
	for all $(k, \ell)\ne(k',\ell')$.
\end{dfn}

\begin{dfn}\label{def:j}
	Given a $1$-dimensional orthonormal wavelet $\psi(x)$
	and constants $C_1, C_2, \beta>0$ and $0<p<2$, define
	\begin{align*}
		\mathcal{J}_\psi^p(C_1, C_2, \beta)
		:=\left\{
			\sum_{k\ge0,\, 0\le\ell<2^k}
			a_{k, \ell}\psi_{k,\ell}
			\,\middle|\,
			\|a\|_{w\ell^p}\le C_1,\ 
			\sum_{k\ge m}a_{k, \ell}^2\le C_22^{-\beta m},\ 
			m=0,1,\ldots
		\right\}.
	\end{align*}
\end{dfn}

For the difinition of multi-dimensional cases,
see Definition \ref{def:j-multi}.

\begin{rem}\label{inclusion}
	If we consider the lexical order on $(k, \ell)$'s,
	then $\mathcal{J}_\psi^p(C_1,C_2,\beta)$ is revealed to be $\beta$-minimally tail compact
	(see Definition \ref{p_tail_comp}).
	Thus, $\mathcal{I}_\phi^p\subset\mathcal{J}_\psi^p$ holds with some modification of constants
	and they have the same degree of sparsity.
\end{rem}

	In the following, we introduce the class $\mathcal{K}^p_\Psi$
	as an expansion of $\mathcal{J}^p_\psi$.
	Though the following definition is treating the $d$-dimensional case,
	considering only the case $d=1$ is sufficient for readers to understand the essential properties.
	For multi-dimensional cases, see also Section \ref{multi-wave}.

\begin{dfn}\label{def:k}
	Let $\Psi\subset L^2([0, 1]^d)$ consist of orthonormal wavelets.
	Then, for an integer $n_s>0$ and constants $C_1, C_2, C_3, \beta>0$ and $0<p<2$,
	define
	\begin{align*}
		\mathcal{K}_\Psi^p(n_s, C_1, C_2, C_3, \beta)
		:=\left\{
			\sum_{j=1}^{n_s}f_j(A_j\cdot -b_j)
			\,\middle|\,
			\begin{array}{c}
				A_j\in\R^{d\times d},\ b_j\in\R^d,\ 
				|\det A_j|^{-1}, \|A_j\|_\infty,
				\|b_j\|_\infty\le C_3,\\
				f_j\in\mathcal{J}^p_{\psi_j}(C_1, C_2, \beta),\ 
				\psi_j\in\Psi,\ j=1, \ldots, n_s
			\end{array}
		\right\}.
	\end{align*}
\end{dfn}

\begin{rem}
	By Remark \ref{inclusion} (and Remark \ref{inclusion-multi})
	the bound given in Theorem \ref{thm:p>0}
	is also the minimax lower bound for $\mathcal{K}_\Psi^p$.
	Moreover,
	$J_k(C)$, introduced in Section \ref{chap:3},
	is included in $K^p_\Psi$, with $k\le n_s$ and a specific $\Psi$ such as one
	containing the Haar wavelet.
	Thus, the bounds given in the previous sections are still applicable to this function class.
	Indeed, as described in Figure \ref{fig:rev2},
	we define $\mathcal{K}^p_\Psi$ as an expansion of $\mathcal{I}^0_\Psi$ as well.
\end{rem}

\subsection{Suboptimlaity of linear estimators on $\mathcal{J}_\psi^p$ and $\mathcal{K}_\psi^p$}

For a wavelet $\psi$ with compact support,
linear estimator shows suboptimality on $\mathcal{J}_\psi^p$ and $\mathcal{K}_\psi^p$,
as is shown in Theorem \ref{linear_not_efficient}.
This can be regarded as a stronger version of Corollary \ref{cor:jk}
and shows the non-efficiency of linear estimators in sparse classes.
The proof of this theorem is similar to that of Theorem 1 in \acite{zhang2002wavelet}
and is given in the appendix (Section \ref{ap-35}).

\begin{thm}\label{linear_not_efficient}
	Let $d=1$ and $\psi$ be a bounded and compactly supported wavelet.
	For any constants $C_1, C_2, \beta>0$ and $0<p<2$,
	there exists a constant $C$ dependent only on $C_1$, $C_2$, and $\beta$
	such that
	\[
		\inf_{\hat{f}:\mathrm{linear}}
		\sup_{f\in \mathcal{J}_\psi^p(C_1, C_2, \beta)}
		R(\hat{f}, f^\circ)
		\ge Cn^{-\frac\beta{1+\beta}}
	\]
	holds for each $n\ge1$.
	This bound also holds on $\mathcal{K}_\Psi^p$.
\end{thm}

\begin{rem}\label{rem:nolinear}
	From this result,
	we see that the minimax-optimal rate
	for linear estimators can be {\it arbitrarily slow} even with the same 
	sparsity $p$, i.e., with a bounded covering entropy (by Lemma \ref{lem:sp1}).
	The nearly optimal rates attained
	by deep learning given in Section \ref{chap:5}
	are unfortunately limited to the case $\beta>1$ (because of the assumption of boundedness),
	but this still serves as evidence for the non-effectiveness of linear methods
	in estimating sparse classes.
\end{rem}

\section{Learning ability of deep ReLU neural networks}\label{chap:5}


\subsection{Mathematical formulation of deep ReLU neural networks}

For mathematical treatments of neural networks,
we have referenced some recent papers on approximation theory
and estimation theory 
\cite{suzuki2018adaptivity,schmidt2017nonparametric,yarotsky2017error,bolcskei2017optimal,keiper2017dgd}.
In the following,
we define neural networks mathematically
and evaluate their covering entropies.

\begin{dfn}
	Let $\rho:\R\to\R$.
	For $L, S, D\in\Z_{>0}$ and $B\ge1$ (with $D\ge d$), define
	$\N(L, S, D, B)$ as the set of all functions $f:\R^d\to\R$ of the form
	\[
		f=W_{L+1}\circ\rho(W_L\cdot-v_L)\circ\cdots\circ\rho(W_1\cdot-v_1),
	\]
	satisfying
	\begin{align*}
		W_1\in\R^{D\times d},\  W_2,\ldots,W_L\in\R^{D\times D},\ 
		W_{L+1}\in\R^{1\times D},\ 
		v_1\in\R^d,\ v_2,\ldots, v_L\in\R^D
	\end{align*}
	and
	\begin{align*}
		\|v_i\|_\infty,\|W_i\|_\infty\le B,\ 
		\sum_{i=1}^{L+1}\|W_i\|_0+\sum_{i=1}^L\|v_i\|_0\le S,
	\end{align*}
	where
	$\rho$ is operated elementwise,
	and 
	$L$, $S$, and $D$ denote the number of hidden layers,
	the sparsity, and the dimensionality of the layers, respectively.
	Also, for $F>0$, we consider a function class
	\begin{align*}
		\N_F=\N_F(L, S, D, B)
		:=
		\{
			\sgn(f)\min\{|f|, F\}
			\mid
			f\in \N(L, S, D, B)
		\}.
	\end{align*}
\end{dfn}

\subsection{Complexity of neural network models}

Hereinafter, we use the ReLU activation function $\rho(x)=\max\{x, 0\}$.
Notice that $\N_F$ can be realized easily using ReLU activation
after an element of $\N$ is computed.
To evaluate the generalization error of deep learning,
we need bounds for the complexity of neural network models.
The following lemma gives the required bound as a function of parameters $L, S, D, B$.

\begin{lem}{\rm\cite{schmidt2017nonparametric,suzuki2018adaptivity}}\label{cov_eval}
	For any $0<\delta<1$,
	the $\delta$-covering entropy with respect to $\|\cdot\|_{L^\infty}$ of $\N(L,S,D,B)$
	(limiting the domain to $[0, 1]^d$)
	can be bounded as
	\[
		V_{(\N(L, S, D, B), \|\cdot\|_{L^\infty})}(\delta)
		\le 2S(L+1)\log\left(\frac{B(L+1)(D+1)}\delta\right).
	\]
\end{lem}

The proof for this lemma is given in the appendix (Section \ref{ap-4}).
We next introduce a lemma which evaluates the approximation ability
of a neural network given the approximation ability of subnetworks.
This lemma is also stated in \acite{bolcskei2017optimal} in another form.
Thanks to this lemma, we can only consider the approximation of the wavelet basis,
where we can exploit many existing studies.

\begin{lem}\label{app_eval}
	Let $\Phi\subset L^2([0, 1]^d)$ satisfy $\|\phi\|_{L^2}=1$ for each $\phi\in\Phi$.
	Suppose for any $\phi\in\Phi$ there exists a function $g\in\N(L, S, D, B)$ such that
	$\|g-\phi\|_{L^2}\le\ve$.
	Then for any $f^\circ\in\mathcal{I}^0_\Phi(n_s, C)$,
	there exists a function
	\[
		f\in\N(L+2, n_s(S+2Dd+d^2+d+1), n_sD, \max\{B, C\})
	\]
	such that
	$\|f-f^\circ\|_{L^2}\le C^{3/2}n_s\ve$
	holds.
\end{lem}

\begin{proof}
	The approximation of
	$f^\circ(x)=\sum_{i=1}^{n_s}c_i\phi_i(A_ix-b_i)\in\mathcal{I}_\Phi^0(n_s, C)$
	can be constructed as shown in Fig. \ref{fig:nn1}
	(we use the ReLU activation function, and so we compute
	$\max\{\tilde{\phi}_i, 0\}$ and $\max\{-\tilde{\phi}_i, 0\}$
	and combine them afterward),
	where each $\tilde{\phi}_i$ approximates $\phi_i$ with an $L^2$-error of at least $\ve$.
	In this construction,
	\begin{align*}
		\|\tilde{f}-f^\circ\|_{L^2}
		\le\sum_{i=1}^{n_s}|c_i|
		\|\phi_i(A_i\cdot-b_i)-\tilde{\phi}_i(A_i\cdot-b_i)\|_{L^2}
		\le \sum_{i=1}^{n_s}C|\det A_i|^{-1/2}\|\tilde{\phi}_i-\phi_i\|_{L^2}
		\le C^{3/2}n_s\ve
	\end{align*}
	holds.
\end{proof}

\begin{rem}\label{rem:def_ap}
	These lemmas evaluate the covering entropy and approximation ability of neural networks.
	Given these two informations,
	we can exploit Theorem \ref{gen_eval} to estimate the generalization error of deep learning by a fixed size of neural network.
	However, we should know how large the parameters $L, S, D, B$ get as $\ve$ goes to zero (or the sample size $n$ goes to infinity).
	As one can easily see, a shallow ReLU network can approximate piece-wise constant functions easily
	(more precisely, see e.g., Fig. \ref{fig:nn2} and the proof of Corollary \ref{cor:trivial}).
	Moreover, deep ReLU networks can approximate polynomials very efficiently \cite{yarotsky2017error}.
	Combining these, for example, we can state that
		\begin{quote}
			there exists an architecture $\N(L, S, D, B)$
			with $L, S, D = \mathrm{O}(\log(1/\ve))$ and $B=\mathrm{O}(1/\ve)$
			that can approximate piece-wise polynomials with $\mathrm{O}(\ve)$ $L^2$-error.
		\end{quote}
	Rigorously speaking, we of course have to restrict the function class (the number of non-smooth points, the degree of polynomials,
	the magnitude of coefficients, etc).
	However, as the main concern of this paper is statistical viewpoints,
	we do not get deeper.
	The important thing is that the above size of network architecture is reasonable
	to approximate ``cheap" functions.
\end{rem}

Motivated by the above remark,
we define the function class $\mathrm{AP}$,
which can be approximated by ``light'' networks, as follows.

\begin{dfn}
	For $C_1, C_2>0$, define $\mathrm{AP}(C_1, C_2)$ as 
	the set of all functions $\phi\in L^2([0, 1]^d)$ satisfying that,
	for each $0<\ve<1/2$,
	there exist parameters $L_\ve, S_\ve, D_\ve, B_\ve>0$ such that 
	\begin{itemize}
		\item
			$L_\ve, S_\ve, D_\ve\le C_1\log(1/\ve)$ and $B_\ve\le C_2/\ve$ hold;
		\item
			there exists a $\tilde\phi \in\N(L_\ve, S_\ve, D_\ve, B_\ve)$
			such that $\|\tilde\phi - \phi\|_{L^2}\le\ve$.
	\end{itemize}
\end{dfn}


\subsection{Generalization ability for an extreme case ($\mathcal{I}_\Phi^0$)}\label{revsec1}

\begin{figure}
	\begin{minipage}{0.49\hsize}
		\centering
		\includegraphics[width=0.98\hsize]{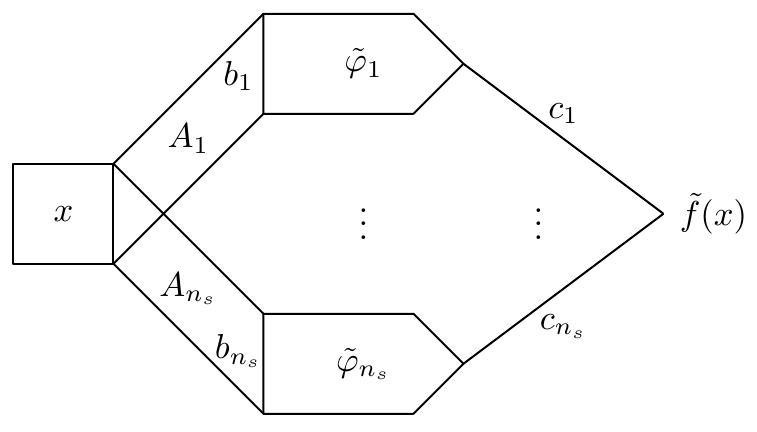}
		\caption{Construction of a larger neural network}
		\label{fig:nn1}
	\end{minipage}
	\begin{minipage}{0.49\hsize}
		\centering
		\includegraphics[width=0.8\hsize]{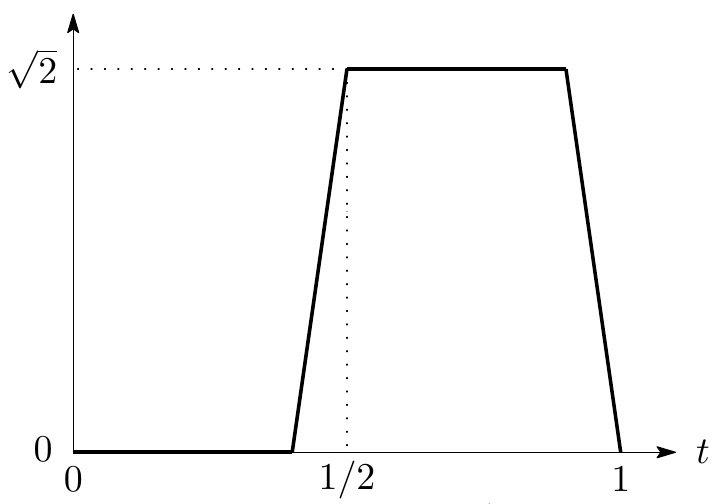}
		\vspace{1mm}
		\caption{An approximation of $\sqrt{2}\cdot1_{[1/2, 1]}$}
		\label{fig:nn2}
	\end{minipage}
\end{figure}

The learning ability of neural networks over $\mathcal{I}_\Phi^0$ is shown in the following.
This is the most extreme case
in terms of the difference between the performance of deep learning and linear methods.

\begin{thm}\label{thm:first_order}
	Let $\Phi\subset L^2([0, 1]^d)$
	satisfy $\|\phi\|_{L^2}=1$
	for each $\phi\in L^2([0, 1]^d)$ and $\sup_{\phi\in\Phi}\|\phi\|_{L^\infty}<\infty$.
	Suppose also $\Phi\subset \mathrm{AP}(C_1, C_2)$ holds for some constants $C_1, C_2>0$.
	Then, for each $n_s, C>0$,
	there exist constants $F, C_3>0$ dependent only on $n_s, C$
	(independent of $n$) such that the empirical risk minimizer $\hat{f}$ over $\N_F^{(n)}$
	satisfies
	\[
		\sup_{f^\circ\in\mathcal{I}_\Phi^0(n_s, C)}R(\hat{f}, f^\circ)\le C_3\frac{(\log n)^3}n
	\]
	for $n\ge2$, where $\N_F^{(n)}$ denotes
	\begin{align*}
		\N_F\left(L_{1/n}+2,
		n_s(S_{1/n}+2D_{1/n}d+d^2+d+1), n_sD_{1/n}, \max\left\{B_{1/n}, C\right\}\right).
	\end{align*}
\end{thm}

\begin{rem}\label{network_size_1}
	This theorem implies that
	the empirical risk minimizer over $\N(L, S, D, B)$ with size
	\[
		L, S, D=\mathrm{O}(\log n),
		\quad
		B=\mathrm{O}(n),
	\]
	almost attains the minimax-optimal rate.
	Therefore,
	from the viewpoint of real-world application,
	a reasonable size of network
	can behave optimally.
\end{rem}

\begin{proof}[Proof of Theorem \ref{thm:first_order}]
	Let $M:=\sup_{\phi\in\Phi}\|\phi\|_{L^\infty}$.
	If we define $F:=n_sCM$,
	each $f^\circ\in\mathcal{I}_\Phi^0(n_s, C)$ satisfies $\|f\|_{L^\infty}\le F$.
	Indeed, $f^\circ$ has some expression
	$f^\circ=\sum_{i=1}^{n_s}c_i\phi_i(A_i\cdot-b_i)$, and so we have
	\[
		\|f^\circ\|_{L^\infty}
		\le\sum_{i=1}^{n_s}|c_i|\|\phi_i\|_{L^\infty}\le n_sCM=F.
	\]
	Hence, for $f\in\N(L, S, D, B)$ and $f^\circ\in\mathcal{I}_\Phi^0(n_s, C)$,
	$\tilde{f}:=\sgn(f)\min\{|f|, F\}$ satisfies
	\begin{align*}
		\|\tilde{f}-f^\circ\|_{L^2}^2
		&=\int_{[0, 1]^d}(\tilde{f}(x)-f^\circ(x))^2\dd x \\
		&\le\int_{[0, 1]^d}(f(x)-f^\circ(x))^2 \dd x
		\tag{$\because$ $f^\circ(x)\in[-F, F]$}\\
		&=\|f-f^\circ\|_{L^2}^2.
	\end{align*}
	Also, notice that $\N_F$'s covering entropy is not greater than that of $\N$,
	and so we have, by Lemma \ref{cov_eval} and the assumption of the assertion,
	\[
		V_{(\N_F^{(n)},\|\cdot\|_{L^\infty})}\left(\frac1n\right)
		\le C_0(\log n)^3
	\]
	for some constant $C_0>0$.
	Then, by Theorem \ref{gen_eval} and Lemma \ref{app_eval},
	\begin{align*}
		\sup_{f^\circ\in\mathcal{I}_\Phi^0(n_s, C)}R(\hat{f}, f^\circ)
		\le \frac{4C^{3/2}n_s}n+C'\left(C_0(F^2+\sigma^2)
		\frac{(\log n)^3}n+\frac{F+\sigma}n\right)
		\le C_3\frac{(\log n)^3}n
	\end{align*}
	holds for some $C_3>0$.
\end{proof}

\begin{cor}\label{cor:trivial}
	Let $d=1$.
	For $J_k(C)$ in Definition \ref{def:jk},
	there exist a constant $F>0$ and a sequence of neural networks $(\N^{(n)})_{n=2}^\infty$
	such that the empirical risk minimizer $\hat{f}$ satisfies
	\[
		\sup_{f^\circ\in J_k(C)}R(\hat{f}, f^\circ)
		\le C_3\frac{(\log n)^3}n
	\]
	for some constant $C_3>0$ independent of $n$ and each $n\ge2$. 
\end{cor}

\begin{proof}
	By Theorem \ref{thm:first_order},
	it suffices to show that
	$\phi=\sqrt{2}\cdot1_{[1/2, 1]}$ can be approximated within $\ve$-error in $L^2$
	by a neural network satisfying the condition of Theorem \ref{thm:first_order}.
	This can be actually realized by a shallow network,
	as
	\begin{align*}
		\frac{4\sqrt{2}}\ve\rho\left(t-\frac{1-\ve/2}{2}\right)
		-\frac{4\sqrt{2}}\ve\rho\left(t-\frac12\right)
		-\frac{4\sqrt{2}}\ve\rho\left(t-\left(1-\frac\ve2\right)\right)
		+\frac{4\sqrt{2}}\ve\rho(t-1)
	\end{align*}
	(Fig. \ref{fig:nn2})
	satisfies the desired condition.
\end{proof}

\begin{rem}
	By Corollary \ref{cor:jk}, Theorem \ref{thm:weak_lower_bound}, and Corollary \ref{cor:trivial},
	$J_k(C)$ demonstrates an extreme situation,
	wherein neural network learning attains the optimal rate up to $\log$ factors whereas
	linear methods are suboptimal.
\end{rem}

This result can easily be expanded to the case of $d=2$
(if we properly define $J_k$ for higher dimensions).
In addition, we can treat a set broader than $J_k(C)$
as $\mathcal{I}_\Phi^0$
because smooth functions such as polynomials can
be well approximated by $\mathrm{O}(\log(1/\ve))$ weights
as has been mentioned in Remark \ref{rem:def_ap}.


\subsection{Generalization ability for the wavelet case ($\mathcal{J}_\psi^p$, $\mathcal{K}_\Psi^p$)}\label{revsec2}

Let us consider the case in which the target function class is $\mathcal{K}_\Psi^p$
in Definition \ref{def:k}.
Note that, by Lemma \ref{app_eval},
we only have to consider approximating functions in $\mathcal{J}_\psi^p$
in Definition \ref{def:j} for fixed $n_s$.
We defer the proof of the following main result to the appendix
(Section \ref{proof_of_main_theorem}).

\begin{thm}\label{thm:main_p>0}
	Let $\psi\in L^2([0, 1]^d)$ be an orthonormal wavelet.
	Suppose there exist constants $C_1', C_2'>0$
	such that $\psi\in\mathrm{AP}(C_1', C_2')$ holds.
	Then, for each $C_1, C_2, \beta>0$ and $0<p<2$,
	there exists a constant $C>0$ dependent only on constants $C_1', C_2', p, C_1, C_2,$
	$\beta$
	(independent of $n$) such that the empirical risk minimizer $\hat{f}$ over 
	$\N_F^{(n)}$ (with some network architecture) satisfies
	\[
		\sup_{\footnotesize\begin{array}{c}
		f^\circ\in\mathcal{J}^p_\psi(C_1, C_2, \beta)\\
		\|f^\circ\|_{L^\infty}\le F
		\end{array}}
		R(\hat{f}, f^\circ)
		\le CF^2n^{-\frac{2\alpha}{2\alpha+1}}(\log n)^3
	\]
	for each $F\ge\max\{1, \sigma\}$ and $n\ge2$,
	where $\alpha:=1/p-1/2$.
	Moreover, for a set of wavelets $\Psi\subset \mathrm{AP}(C_1', C_2')$,
	the same evaluation is valid for $\mathcal{K}^p_\Psi$;
	\[
		\sup_{\footnotesize\begin{array}{c}
		f^\circ\in\mathcal{K}^p_\Psi(n_s, C_1, C_2, C_3, \beta)\\
		\|f^\circ\|_{L^\infty}\le F
		\end{array}}
		R(\hat{f}, f^\circ)
		\le C'F^2n^{-\frac{2\alpha}{2\alpha+1}}(\log n)^3
	\]
	holds for some $C'>0$,
	where $\hat{f}$ is the empirical risk minimizer of some larger network.
\end{thm}

\begin{rem}
	The actual network size of $\N(L, S, D, B)$ here is
	\[
		L=\mathrm{O}(\log n),
		\quad
		S, D=\mathrm{O}(n^{\frac1{2\alpha+1}}\log n),
		\quad
		B=\mathrm{O}\left(n^{\max\left\{ 1, \frac{4\alpha}{\beta(2\alpha+1)} \right\}}\right)
	\]
	(see (\ref{eq:nn2}) in the appendix).
	Though this network is larger than that of Remark \ref{network_size_1},
	this is still acceptable size,
	especially with large $\alpha$; i.e., in the case the parameter space is very sparse.
\end{rem}

Concretely, $\psi$ constructed by using
the Haar wavelet satisfies the desired condition.
Also, in the case of $d=1$ and $\beta>1$, we can remove the restriction by the constant $F$,
because $\sup\{\|f\|_\infty\mid f\in J^p_\psi \}<\infty$ holds.
Note also that this result is nearly minimax optimal; i.e.,
$\hat{f}$ attains the minimax lower bound derived in Theorem \ref{thm:p>0} up to
$\log$ factors and the constraint of boundedness.
As $p$ gets smaller, $\mathcal{K}_\Psi^p$ gets sparser
and more reasonable size of neural network can achieve the minimax optimality.
Moreover, as is stated in the next remark,
linear estimators cannot exploit the sparsity of $\mathcal{K}_\Psi^p$
in typical settings (e.g., when $\Psi$ contains non-continuous functions).
From the viewpoint of approximation ability,
this situation can be interpreted as follows;
neural networks are allowed to approximate just the sparse set $\mathcal{K}_\Psi^p$,
whereas linear estimators must be able to approximate all the functions
contained in $\conv(\mathcal{K}_\Psi^p)$.
This difference yields the difference of complexity;
linear estimators require far more parameters than neural networks.
Finally, this situation results in
the great difference of generalization ability between deep and linear.

\begin{rem}
	If $\Psi$ contains the Haar wavelet and $C_3$ is sufficiently large,
	$\mathcal{K}_\Psi^p$ includes $J_k$ in Definition \ref{def:jk}
	(also, notice that $J_k$ is bounded in the $L^\infty$-norm sense).
	To be more precise,
	$\{f\in[0, 1]^d\to\R\mid f(x_1,\ldots,x_d)=g(x_1),\ g\in J_k(C)\}$
	is included in $K_\Psi^p$ for some $C>0$.
	The proof of Theorem \ref{thm:bv} can easily be modified for this case,
	and we have
	\[
		\inf_{\hat{f}:\text{linear}}\sup_{f^\circ\in\mathcal{K}_\Psi^p
		,\:\|f^\circ\|_\infty\le F}R(\hat{f}, f^\circ)
		\ge cn^{-1/2}
	\]
	for some $c>0$.
	If $\alpha>1/2$ (equivalent to $p<1$) holds,
	then the neural network learning is superior to linear methods.
\end{rem}


\subsection{Parameter sharing technique to restrict the covering entropy}\label{sec:p_share}

The assumption of the ability for $\varphi$ to be approximated by $\tilde{\varphi}$ imposed
in Theorems \ref{thm:first_order} and \ref{thm:main_p>0} is quite strong,
and thus we cannot treat a broad range of wavelets.
In the proof of Theorem \ref{thm:main_p>0}, however,
we do not exploit the full degree of freedom depicted in Fig. \ref{fig:nn1}
because subnetworks share the same approximator $\tilde{\psi}$.
In this subsection, we consider neural networks with parameter sharing.

\begin{dfn}
	Let $N$ be a positive integer.
	For a given neural network architecture $\N(L, S, D, B)$,
	denote the {\it $N$-sharing} of $\N(L, S, D, B)$ by
	$\N^N(L, S, D, B)$, defined as
	\[	
		\left\{
			\sum_{i=1}^Nc_if(A_i\cdot-b_i)
			\,\middle|\,
				A_i\in\R^{d\times d},\ b_i\in\R^d,\ c_i\in\R, \
				\|A_i\|_\infty, \|b_i\|_\infty, |c_i|\le B,\
				i=1,\ldots, d, \
				f\in\N(L, S, D, B)
		\right\}.
	\]
\end{dfn}

\begin{thm}\label{p_share_own}
	Given a positive integer $N$ and $\mathcal{N}(L, S, D, B)$ with $L\ge2$,
	the $\delta$-covering entropy with respect to $\|\cdot\|_{L^\infty}$ of $\N^N(L,S,D,B)$
	(limiting the domain to $[0, 1]^d$)
	can be bounded as
	\begin{align*}
		V_{(\N^N(L, S, D, B), \|\cdot\|_{L^\infty})}(\delta)
		\le \left(N(d+1)^2+2S(L+1)\right)(L+3)
		\log\left(\frac{NB(L+1)(D+1)}\delta\right)
	\end{align*}
	for any $0<\delta<1$.
\end{thm}

The proof is straightforward but a bit technical; thus,
we defer it to the appendix (Section \ref{ap-5}).

\begin{rem}
	If we use neural networks with parameter sharing,
	we can use non-trivial wavelets with some smoothness;
	i.e., in Theorem \ref{thm:main_p>0},
	the assumption for $\psi$ can be weakened to the following:
	\begin{quote}
		For each $0<\ve<1/2$, there exist $L_\ve, S_\ve, D_\ve, B_\ve>0$ satisfying
		\begin{itemize}
			\item
				$L_\ve\le C_1'\log(1/\ve)$,
				$D_\ve, B_\ve\le C_2'\ve^{-\gamma}$ and
				$S_\ve \le C_3'\ve^{-\frac1{\alpha+1}}$ hold
				for some constants $C_1', C_2', C_3', \gamma>0$;
			\item
				there exists
				$\tilde{\psi}\in\N(L_\ve, S_\ve, D_\ve, B_\ve)$
				such that $\|\tilde{\psi}-\psi\|_{L^2}\le\ve$.
		\end{itemize}
	\end{quote}
	This class of $\psi$ is actually broadened as there exist compactly supported wavelets with high regularity
	(a large H\"{o}lder exponent)
	\cite{ingrid1992ten},
	and such functions can be approximated well by networks with a small number of parameters.
	Indeed, \acite{yarotsky2017error} proved that
	\begin{quote}
		A unit ball of Sobolev space $W^{n, \infty}([0, 1]^d)$
		can be approximated, with $L^\infty$-error at most $\ve$,
		by a neural network with
		$\mathrm{O}(\log(1/\ve))$ depth,
		and $\mathrm{O}(\ve^{-d/n})$ complexity (corresponds to $S, D$).
	\end{quote}
	Though we cannot use this in the original form 
	because $L^\infty$ should be replaced by $L^2$ and the upper bound of $B$ should be given,
	the refinement is not so difficult \cite[e.g., ][]{schmidt2017nonparametric}.
\end{rem}

\section{Summary and discussion}\label{chap:6}


\subsection{Summary}

In this paper,
we have shown that deep learning
outperforms other commonly used methods such as linear estimators even in a simple case.
To evaluate the learning ability of estimators,
we employed a Gaussian regression problem with a sparse target function space.
In such a problem setting,
neural network learning attains nearly the minimax-optimal rate of
convergence with respect to the sample size,
whereas a linear estimator can only achieve a suboptimal rate.
The main novelty is that the target function spaces
were selected to have natural sparsity, instead of following the well-known settings
developed by the existing mathematical analyses.
We have also shown that parameter sharing is quite effective
for widening function classes where (near) minimax optimality holds.


\subsection{Discussion and future work}

There are two main limitations in this work that remain to be addressed in future investigations.

First, $\mathcal{I}^0_\Phi(n_s, C)$ is the most extreme case
in the sense that deep learning outperforms linear estimators.
This class is very simple, and we have additionally defined $w\ell^p$-bounded classes
$\mathcal{K}^p_\Psi$ for $0<p<2$,
with the assumption of orthonormal wavelets.
However, we should remove the orthogonality if we follow the philosophy of defining $\mathcal{I}^0_\Phi(n_s, C)$.
For example, using the definition
\begin{align*}
	\mathcal{L}_\Phi^p(C_1, C_2, C_3, \beta)
	:=\left\{
		\sum_{i=1}^\infty
		c_i\phi(A_i\cdot-b_i)
		\,\middle|\,
		\begin{array}{c}
			|c_i|\le C_1i^{-1/p},\ 
			\|A_i\|_\infty, \|b_i\|_\infty \le C_2i^\beta, \\
			|\det A_i|^{-1}\le C_3,\ 
			i=1,2,\ldots,\ \phi\in\Phi
		\end{array}
	\right\}
\end{align*}
would be one possible way.
Of course, for some range of $(p, \beta)$,
we can show that deep learning attains a rate faster than do linear estimators.
However, we could not have shown that the convergence rate satisfies minimax optimality,
even up to log factors.
This difficulty arises from the fact that we have fully exploited the orthogonality 
in the proof of deep learning's minimax optimality over $\mathcal{K}^p_\Psi$.
It is possible that we can find both a better minimax lower bound for $\mathcal{L}_\Phi^p$
and a better approximation bound by neural networks.

Second, parameter sharing, mentioned in Subsection \ref{sec:p_share}, is used mainly in the context of convolutional neural networks (CNNs),
and this implies the superiority of CNNs in solving a regression problem.
However, some of the arguments in this paper are not directly applicable to the analysis of CNNs.
Although CNNs have achieved notable success in pattern recognition,
theories of CNNs with respect to regression problems have not yet been well argued
in the literature.

In addition to these issues,
a theoretical analysis of stochastic optimization as used in deep learning
is needed,
which is not treated in this paper.


\section*{Acknowledgments}
The authors would like to thank Atsushi Nitanda for his constructive feedback.
TS was partially supported by MEXT Kakenhi (15H05707, 18K19793 and 18H03201),
Japan Digital Design, and JST-CREST.

\bibliographystyle{model2-names.bst}
\bibliography{cite}

\appendix

\section{Mathematical supplements}

\subsection{Generalities related to the minimax risk}

In the Gaussian regression model,
for two true functions $f, g\in\F^\circ$,
it is well known that
the Kullback--Leibler (KL) divergence between two distributions of $(X, Y)$ generated by $f$ and $g$
is easy to calculate. 
Let $d_\mathrm{KL}(f, g)$ be the square root of the KL divergence.
The following lemma is essential when one treats a regression problem as
a parameter estimation problem
\cite{yang1999information,schmidt2017nonparametric,suzuki2018adaptivity};
its proof is given in Section \ref{ap-1}.

\begin{lem}\label{KL-L2}
	It holds that $\displaystyle d_\mathrm{KL}(f, g)^2=\frac1{2\sigma^2}\|f-g\|_{L^2}^2$.
\end{lem}

Therefore, the square root of the KL divergence is
a metric equivalent to the $L^2$ metric in regression problems with a Gaussian noise.

We also introduce a lemma
that is useful for deriving a lower bound of the metric entropy.

\begin{lem}{\rm\cite[Lemma 4]{donoho1993unconditional}}\label{lem:sp2}
	Let $\mathcal{C}_k\subset\ell^2$ be a $k$-dimensional hypercube of side $2\delta>0$ defined as
	\[
		\mathcal{C}_k:=\{a\in\ell^2\mid
		|a_1|, \ldots, |a_k|\le\delta,\ |a_{k+1}|=|a_{k+2}|=\cdots=0
		\}.
	\]
	Then there exists a constant $A>0$ such that
	\[
		V_{(\mathcal{C}_k, \|\cdot\|_{\ell^2})}\left(\frac{\delta\sqrt{k}}2\right)\ge Ak,
		\quad k=1,2,\ldots.
	\]
\end{lem}


\subsection{Treatment of multi-dimensional wavelets}\label{multi-wave}

In this section,
we give a multi-dimensional extension of arguments in Section \ref{sec:sparse_wavelet}.
Here, we adopt the multiplication of elementwise wavelets as a multi-dimensional wavelet.

\begin{lem}
	For orthogonal wavelets $\psi^{(1)}, \ldots, \psi^{(d)}$,
	\[
		\psi(x):=\prod_{i=1}^d\psi^{(i)}(x_i),\quad
		x=(x_1,\ldots,x_d)\in \R^d,
	\]
	is a $d$-dimensional orthogonal wavelet;
	i.e.,
	\[
		\left\{
			\prod_{i=1}^d\psi_{k_i, \ell_i}^{(i)}
			\,\middle|\,
			k_i\ge0,\ 0\le\ell_i<2^{k_i},\ i=1,\ldots,d
		\right\}
	\]
	is an orthonormal subset of $L^2([0, 1]^d)$.
\end{lem}

\begin{proof}
	The normality is clear by Fubini's theorem.
	Also, for distinct wavelets $\psi, \psi'$ in the set,
	there exists $i$ such that $(k_i, \ell_i)\ne(k_i', \ell_i')$.
	Since we have $\psi\psi'\in L^1([0, 1]^d)$ by the AM-GM inequality,
	Fubini's theorem leads to the conclusion.
\end{proof}

\begin{dfn}\label{def:j-multi}
	Given an orthonormal wavelet
	\[
		\psi(x)=\psi^{(1)}(x_1)\cdots\psi^{(d)}(x_d)
	\]
	and constants $C_1, C_2, \beta>0$ and $0<p<2$, define
	\begin{align*}
		\mathcal{J}_\psi^p(C_1, C_2, \beta)
		:=\left\{
			\sum_{(k, \ell)\in T_0}
			a_{k, \ell}\psi_{k,\ell}
			\,\middle|\,
			\|a\|_{w\ell^p}\le C_1,\ 
			\sum_{(k, \ell)\in T_m}a_{k, \ell}^2\le C_22^{-\beta m},\ 
			m=0,1,\ldots
		\right\},
	\end{align*}
	where the sets $T_m$ ($m=0,1,\ldots$) are defined as
	\[
		T_m:=\left\{
			(k,\ell)
			\in\Z^d\times\Z^d
			\,\middle|\,
			\begin{array}{c}
				k=(k_1,\ldots,k_d),\ 
				\ell=(\ell_1,\ldots,\ell_d),\\
				k_i\ge0,\ 0\le\ell_i<2^{k_i},\ 
				i=1,\ldots,d, \\
				\max_{1\le i\le d}k_i\ge m
			\end{array}
		\right\},
	\]
	and $\psi_{k, \ell}$ denotes
	$\displaystyle\prod_{i=1}^d\psi_{k_i, \ell_i}^{(i)}$
	for each
	\[
		(k, \ell)
		=\bigl((k_1,\ldots,k_d), (\ell_1,\ldots,\ell_d)\bigr)\in T_0.
	\]
\end{dfn}

\begin{rem}\label{inclusion-multi}
	If we define a partial order on $S_0$ by $(k, \ell)\preceq (k', \ell')
	\Leftrightarrow \max_ik_i\le \max_ik_i'$
	and then sort it,
	$\mathcal{J}_\psi^p(C_1,C_2,\beta)$ is revealed to be $\beta/d$-minimally tail compact.
	Thus, $\mathcal{I}_\phi^p\subset\mathcal{J}_\psi^p$ holds with some modification of constants.
\end{rem}


\section{Proofs}

\subsection{Proof of Theorem \ref{gen_eval}}\label{ap-2}
	(mainly following the original proof) 
	First, we evaluate the value of
	\begin{align*}
		D:=\left\lvert\mathrm{E}\left[\frac1n\sum_{i=1}^n
			(\hat{f}(X_i)-f^\circ(X_i))^2
		\right]-R(\hat{f}, f^\circ)\right\rvert.
	\end{align*}
	Let $X_1', \ldots, X_n'$ be i.i.d. random variables generated to be independent of
	$(X_i, Y_i)_{i=1}^n$.
	Then we have
	\begin{align*}
		R(\hat{f}, f^\circ)=\frac1n\sum_{i=1}^n\mathrm{E}\left[
			(\hat{f}(X_i')-f^\circ(X_i'))^2
		\right],
	\end{align*}
	and so we obtain
	\begin{align*}
		D&=\left|
			\mathrm{E}\left[
			\frac1n\sum_{i=1}^n\left((\hat{f}(X_i)-f^\circ(X_i))^2
			-(\hat{f}(X_i')-f^\circ(X_i'))^2\right)
			\right]
		\right|\\
		&\le\frac1n\mathrm{E}
		\left[\left|
			\sum_{i=1}^n
			\left((\hat{f}(X_i)-f^\circ(X_i))^2
			-(\hat{f}(X_i')-f^\circ(X_i'))^2\right)
			\right|\right].
	\end{align*}
	Here, let $G_\delta=\{f_1, \ldots, f_N\}$ be a $\delta$-covering of $\F$
	with the minimum cardinality in the $L^\infty$ metric.
	Notice that $\log N \ge 1$.
	If we define $g_j(x, x'):=(f_j(x)-f^\circ(x))^2-(f_j(x')-f^\circ(x'))^2$ and
	a random variable $J$ taking values in $\{1,\ldots, N\}$ such that
	$\|\hat{f}-f_J\|_{L^\infty}\le\delta$,
	we have
	\begin{align}
		D\le\frac1n\mathrm{E}\left[
			\left|
				\sum_{i=1}^n g_J(X_i, X_i')
			\right|
		\right]+8F\delta.\label{eq:g1}
	\end{align}
	In the above evaluation, we have used the inequality
	\begin{align*}
		\left|(\hat{f}(x)-f^\circ(x))^2-(f_J(x)-f^\circ(x))^2\right|
		=\left|\hat{f}(x)-f_J(x)\right|\left|\hat{f}(x)+f_J(x)-2f^\circ(x)\right|
		\le 4F\delta.
	\end{align*}
	Define constants $r_j:=\max\{A, \|f_j-f^\circ\|_{L^2}\}$ ($j=1,\ldots, N$) and
	a random variable
	\[
		T:=\max_{1\le j\le N}\left|\sum_{i=1}^n\frac{g_j(X_i, X_i')}{r_j}\right|,
	\]
	where $A>0$ is a deterministic quantity fixed afterward.
	Then, because of (\ref{eq:g1}), we have
	\begin{align}
		D\le\frac1n\mathrm{E}[r_JT]+8F\delta
		\le\frac1n\sqrt{\mathrm{E}[r_J^2]\mathrm{E}[T^2]}+8F\delta
		\le \frac12\mathrm{E}[r_J^2]+\frac1{2n^2}\mathrm{E}[T^2]+8F\delta
		\label{eq:g2}
	\end{align}
	by the Cauchy--Schwarz inequality and the AM-GM inequality.
	Here, by the definition of $J$, $\mathrm{E}[r_J^2]$ can be evaluated as follows:
	\begin{align}
		\mathrm{E}[r_J^2]
		\le A^2+\mathrm{E}\left[
			\|f_J-f^\circ\|_{L^2}^2
		\right]
		\le A^2+\mathrm{E}\left[\|\hat{f}-f^\circ\|_{L^2}^2\right]+4F\delta
		=R(\hat{f}, f^\circ)+A^2+4F\delta.\label{eq:g3}
	\end{align}
	Because of the independence of the defined random variables, 
	\begin{align*}
		\mathrm{E}\left[
			\left(
				\sum_{i=1}^n\frac{g_j(X_i, X_i')}{r_j}
			\right)^2
		\right]
		&=\sum_{i=1}^n
		\mathrm{E}\left[\left(
			\frac{g_j(X_i, X_i')}{r_j}
		\right)^2\right]\\
		&=\sum_{i=1}^n
		\biggl(\mathrm{E}\left[
			\frac{(f_j(X_i)-f^\circ(X_i))^4}{r_j^2}
		\right]+\mathrm{E}\left[
			\frac{(f_j(X_i')-f^\circ(X_i'))^4}{r_j^2}
		\right]
		\biggr)\\
		&\le  2F^2n
	\end{align*}
	holds, where we have used the fact that each $g_j(X_i, X_i')$ is centered.
	Then, using Bernstein's inequality, we have, in terms of 
	$r:=\min_{1\le j\le N} r_j$,
	\begin{align*}
		\mathrm{P}(T^2\ge t)
		=\mathrm{P}(T\ge\sqrt{t})
		\le2N\exp\left(
			-\frac{t}{2F^2\left(2n+\frac{\sqrt{t}}{3r}\right)}
		\right),\quad t\ge0.
	\end{align*}
	Let us evaluate $\mathrm{E}[T^2]$.
	For arbitrary $t_0>0$,
	it holds that
	\begin{align*}
		\mathrm{E}[T^2]&=\int_0^\infty P(T^2\ge t)\dd t\\
		&\le t_0+\int_{t_0}^\infty P(T^2\ge t) \dd t\\
		&\le t_0+2N\int_{t_0}^\infty \exp\left(
			-\frac{t}{8F^2n}
		\right)\dd t+2N\int_{t_0}^\infty \exp\left(
		-\frac{3r\sqrt{t}}{4F^2} \right)\dd t.
	\end{align*}
	We compute the values of these two integrals in terms of $t_0$:
	\begin{align*}
		\int_{t_0}^\infty\exp\left(-\frac{t}{8F^2n}\right)\dd t
		&=\left[-8F^2n\exp\left(-\frac{t}{8F^2n}\right)\right]_{t_0}^\infty
		=8F^2n\exp\left(-\frac{t_0}{8F^2n}\right),\\
		\int_{t_0}^\infty \exp\left(
		-\frac{3r\sqrt{t}}{4F^2} \right)\dd t
		&=\int_{t_0}^\infty\exp(-a\sqrt{t})\dd t\quad 
		\tag{$a:=3r/4F^2$}\\
		&=\left[
			-\frac{2(a\sqrt{t}+1)}{a^2}\exp(-a\sqrt{t})
		\right]_{t_0}^\infty\\
		&=\frac{8F^2\sqrt{t_0}}{3r}\exp\left(
		-\frac{3r\sqrt{t_0}}{4F^2} \right)
		+\frac{32F^2}{9r^2}\exp\left(
		-\frac{3r\sqrt{t_0}}{4F^2} \right).
	\end{align*}
	Now we determine $A=\sqrt{t_0}/6n$.
	Since we have $r\ge A=\sqrt{t_0}/6n$,
	\begin{align*}
		\mathrm{E}[T^2]
		&\le
		t_0+2N\left(
			8F^2n+16F^2n+\frac{128F^2n^2}{t_0}
		\right)\exp\left(-\frac{t_0}{8F^2n}\right)\\
		&\le
		t_0+16NF^2n\left(3+\frac{16n}{t_0}\right)\exp\left(-\frac{t_0}{8F^2n}\right)
	\end{align*}
	holds.
	Letting $t_0=8F^2n\log N$, the above evaluation can be rewritten as
	\begin{align}
		\mathrm{E}[T^2]
		\le 8F^2n\left(\log N+6+\frac{2}{F^2\log N}\right).\label{eq:g4}
	\end{align}
	Finally, we combine (\ref{eq:g2}), (\ref{eq:g3}), (\ref{eq:g4}), and
	$\displaystyle A^2=\frac{2F^2\log N}{9n}$ to obtain
	\begin{align*}
		D&\le \left( \frac12R(\hat{f}, f^\circ)+\frac12A^2+2F\delta \right)
		+\frac{4F^2}n\left(\log N+6+\frac2{F^2\log N}\right)
		+8F\delta\\
		&\le \frac12R(\hat{f}, f^\circ)+\frac{F^2}n\left(\frac{37}9\log N+32\right) +10F\delta,
	\end{align*}
	where we have used the fact that $\log N\ge 1$.
	Thus, we obtain the evaluation
	\begin{align}
		R(\hat{f}, f^\circ)
		\le
		2\mathrm{E}\left[\frac1n\sum_{i=1}^n
			(\hat{f}(X_i)-f^\circ(X_i))^2
		\right]
		+\frac{2F^2}n\left(\frac{37}9\log N+32\right) +20F\delta.
		\label{eq:g5}
	\end{align}
	
	Next, we evaluate the quantity
	\begin{align}
		\hat{R}:=\mathrm{E}\left[\frac1n\sum_{i=1}^n
			(\hat{f}(X_i)-f^\circ(X_i))^2
		\right].
		\label{eq:g_rhat}
	\end{align}
	Since $\hat{f}$ is an empirical risk minimizer,
	for arbitrary $f\in\F$,
	\[
		\mathrm{E}
		\left[\frac1n\sum_{i=1}^n
			(\hat{f}(X_i)-Y_i)^2
		\right]
		\le\mathrm{E}
		\left[\frac1n\sum_{i=1}^n
			(f(X_i)-Y_i)^2
		\right]
	\]
	holds.
	As $Y_i=f^\circ(X_i)+\xi_i$,
	we have
	\begin{align*}
		&\mathrm{E}\left[
			(f(X_i)-Y_i)^2
		\right]-
		\mathrm{E}\left[
			(\hat{f}(X_i)-Y_i)^2
		\right]\\
		&=\mathrm{E}\left[
			(f(X_i)-f^\circ(X_i))^2
		\right]
		-2\mathrm{E}\left[\xi_if(X_i)\right]-\mathrm{E}\left[
			(\hat{f}(X_i)-f^\circ(X_i))^2
		\right]
		+2\mathrm{E}\left[\xi_i\hat{f}(X_i)\right]\\
		&=\left(
			\|f-f^\circ\|_{L^2}^2
			+2\mathrm{E}\left[\xi_i\hat{f}(X_i)\right]
		\right)-\mathrm{E}\left[
			(\hat{f}(X_i)-f^\circ(X_i))^2
		\right].
	\end{align*}
	Here we have used the fact that
	\[
		\mathrm{E}[\xi_if(X_i)]
		=\mathrm{E}[\xi_i]\mathrm{E}[f(X_i)]=0
	\]
	holds because of the independence between $\xi_i$ and $X_i$ and the fact that both $\xi_i$ and
	$f(X_i)$ have a finite $L^1$ norm.
	Thus we have
	\begin{align}
		\hat{R}
		\le \|f-f^\circ\|_{L^2}^2+\mathrm{E}
		\left[\frac2n\sum_{i=1}^n\xi_i\hat{f}(X_i)\right].
		\label{eq:g6}
	\end{align}
	Let us evaluate the second term on the right-hand side.
	\begin{align}
		\left|\mathrm{E}\left[
			\frac2n\sum_{i=1}^n\xi_i\hat{f}(X_i)
		\right]\right|
		&=
		\left|\mathrm{E}\left[
			\frac2n\sum_{i=1}^n\xi_i(\hat{f}(X_i)-f^\circ(X_i))
		\right]\right|\nonumber\\
		&\le\frac{2\delta}n\mathrm{E}\left[
			\sum_{i=1}^n|\xi_i|
		\right]+
		\left|\mathrm{E}\left[
			\frac2n\sum_{i=1}^n\xi_i(f_J(X_i)-f^\circ(X_i))
		\right]\right|.
		\label{eq:g7}
	\end{align}
	Here, the first term is upper-bounded by applying the Cauchy--Schwarz inequality:
	\begin{align}
		\frac{2\delta}n\mathrm{E}\left[\sum_{i=1}^n|\xi_i|\right]
		\le\frac{2\delta}n\mathrm{E}\left[n^{1/2}\left(\sum_{i=1}^n\xi_i^2\right)^{1/2}\right]
		\le\frac{2\delta}{\sqrt{n}}\mathrm{E}\left[\sum_{i=1}^n\xi_i^2\right]^{1/2}
		=2\sigma\delta.
		\label{eq:g8}
	\end{align}
	Let $\ve_j$ $(j=1,\ldots, N)$ be random variables defined as
	\[
		\ve_j:=\frac{\sum_{i=1}^n\xi_i(f_j(X_i)-f^\circ(X_i))
		}{\bigl(\sum_{i=1}^n(f_j(X_i)-f^\circ(X_i))^2\bigr)^{1/2}},
	\]
	where $\ve_j:=0$ if the denominator equals $0$.
	Notice that each $\ve_j$ follows a centered Gaussian distribution with variance $\sigma^2$
	(conditional on $X_1,\ldots,X_n$).
	Now we have, using the Cauchy--Schwarz inequality and the AM-GM inequality,
	\begin{align}
		\left|\mathrm{E}\left[
			\frac2n\sum_{i=1}^n\xi_i(f_J(X_i)-f^\circ(X_i))
		\right]\right|
		&=
		\frac2n\left|\mathrm{E}\left[
			\left(\sum_{i=1}^n(f_J(X_i)-f^\circ(X_i))^2\right)^{1/2}\ve_J
		\right]\right|\nonumber\\
		&\le\frac2{\sqrt{n}}\mathrm{E}\left[
			\frac1n\sum_{i=1}^n(f_J(X_i)-f^\circ(X_i))^2
		\right]^{1/2}\!\!\!\mathrm{E}\left[\max_{1\le j\le N}\ve_j^2\right]^{1/2}\nonumber\\
		&\le\frac2{\sqrt{n}}\sqrt{\hat{R}+4F\delta}\
		\mathrm{E}\left[\max_{1\le j\le N}\ve_j^2\right]^{1/2}\nonumber\\
		&\le\frac12(\hat{R}+4F\delta)+\frac2n
		\mathrm{E}\left[\max_{1\le j\le N}\ve_j^2\right].
		\label{eq:g9}
	\end{align}
	By a similar argument as given in the proof of \acite[Theorem 7.47]{cmu},
	for any $0<t<1/2\sigma^2$,
	\begin{align*}
		\exp\left(
			t\mathrm{E}\left[\max_{1\le j\le N}\ve_j^2\right]
		\right)
		&\le \mathrm{E}\left[
			\max_{1\le j\le N}\exp\left(t \ve_j^2\right)
		\right]
		\tag{by Jensen's inequality}
		\\
		&\le
		N \mathrm{E}\left[
			\exp\left(t\ve_1^2\right)
		\right]\\
		&=\frac{N}{\sqrt{2\pi\sigma^2}}\int_{-\infty}^\infty
		e^{tx^2}e^{-\frac{x^2}{2\sigma^2}}\dd x
		=\frac{N}{\sqrt{1-2\sigma^2 t}}
	\end{align*}
	holds. Therefore we have, by determining $t=1/4\sigma^2$,
	\begin{align}
		\mathrm{E}\left[\max_{1\le j\le N}\ve_j^2\right]
		\le 4\sigma^2\log(\sqrt{2}N)\le 4\sigma^2(\log N+1).
		\label{eq:g10}
	\end{align}
	Now we combine (\ref{eq:g6})--(\ref{eq:g10}) to obtain
	\[
		\hat{R}\le\|f-f^\circ\|_{L^2}^2+2\sigma\delta+\frac12(\hat{R}+4F\delta)
		+\frac{8\sigma^2}n(\log N+1),
	\]
	and so
	\begin{align}
		\hat{R}\le2\|f-f^\circ\|_{L^2}^2+4(\sigma+F)\delta+\frac{16\sigma^2}n(\log N+1)
		\label{eq:g11}
	\end{align}
	holds.
	
	Finally, since $f$ is an arbitrary element of $\F$,
	we combine (\ref{eq:g5}), (\ref{eq:g_rhat}), and (\ref{eq:g11}) to have
	\begin{align*}
		R(\hat{f}, f^\circ)
		&\le 4\inf_{f\in\F}\|f-f^\circ\|_{L^2}^2
		+
		\frac1n\left(
			\left(
				\frac{37}9F^2+32\sigma^2
			\right)\log N
			+32(F^2+\sigma^2)
		\right)+(18F+8\sigma)\delta,
	\end{align*}
	and this leads to the conclusion.


\subsection{Proof of Theorem \ref{thm:l2}}\label{convproof}

We first prove the following lemma.

\begin{lem}\label{Fatou}
	For any affine estimator $\hat{f}$ and a sequence
	$f^\circ_1, f^\circ_2,\ldots\in L^2([0, 1]^d)$
	convergent to $f_\infty^\circ\in L^2([0, 1]^d)$ almost everywhere,
	\[
		R(\hat{f}, f_\infty^\circ)\le\sup_{m\ge1}R(\hat{f}, f_m^\circ)
	\]
	holds.
\end{lem}

\begin{proof}
	The conclusion follows immediately from Fatou's lemma and Eqs.\ (\ref{eq:l1}) and (\ref{eq:l2}).
\end{proof}

\begin{rem}
	In the proof of Lemma \ref{Fatou},
	we have not used the linearity of $\hat{f}$.
	Indeed, if $\hat{f}(X)-f^\circ_m(X)$ is convergent to $\hat{f}(X)-f^\circ_\infty(X)$
	with probability $1$, then we have the same conclusion (where $X$ is a uniformly distributed
	random variable independent of other observed random variables).
	Hence, Lemma \ref{Fatou} is applicable to
	a broader class of estimators, such as estimators continuous with respect to observed data
	in some metric.
\end{rem}

The assertion of Theorem \ref{thm:l2} is now clear from Eq. (\ref{l_add}), Lemma \ref{Fatou},
and the fact that a sequence convergent to a function in $L^2$ has
a subsequence that is convergent to the same function almost everywhere.


\subsection{Proof of Lemma \ref{lem:ll}}\label{rev_lem:ll}

The following lemma is a well-known property of functions of bounded total variation.

\begin{lem}{\rm\cite{stein2005real}}\label{cor:tv1}
	For each function $f:[0, 1]\to\R$ with $TV^f([0,1])<\infty$,
	there exist increasing $f^+, f^-:[0,1]\to\R$ such that
	\begin{align*}
		f=f^+-f^-+f(0),\quad
		f^+(0)=f^-(0)=0,\quad
		f^+(t)+f^-(t)=TV^f([0, t]), \quad 0\le t\le 1,
	\end{align*}
	holds.
\end{lem}

By using this lemma, Lemma \ref{lem:ll} can be proven as follows.

	Since $J_1(C)\subset J_k(C)$ holds for each $k$,
	it suffices to show the assertion for $k=1$.
	By the definition of convex hull, we have
	\begin{align*}
		\conv(J_1(C))
		=\left\{
			a_0+\sum_{i=1}^ka_i1_{[t_i, 1]}
			\,\middle|\,
			t_i\in(0, 1],\ |a_0|\le C,\ 
			\sum_{i=1}^k|a_i|\le C,\ k\ge1
		\right\}
		=\bigcup_{k=1}^\infty J_k(C).
	\end{align*}
	It is obvious that $J_k(C)\subset BV(C)$ for each $k$.
	Thus, we have only to show that for each $f\in BV(C)$ and $\ve>0$,
	there exist some $k\ge1$ and $f_k\in J_k(C)$ such that
	$\|f_k-f\|_{L^2}\le\ve$ holds.
	
	Let $f\in BV(C)$, and take $f^+$ and $f^-$ satisfying the condition of
	Lemma \ref{cor:tv1}.
	Then $f$ can be written as $f=a_0+f^+-f^-$, where $a_0:=f(0)$ is a constant.
	Let $g:[0, 1]\to\R$ be an increasing function satisfying $g(0)=0$,
	and define
	\[
		g_k:=\sum_{i=1}^{k}\left(g\left(\frac{i}{k}\right)
		-g\left(\frac{i-1}{k}\right)\right)1_{[i/k, 1]}.
	\]
	Then we have $g_k\in J_k(g(1))$ and $g_k(t)=g(i/k)$ for $t\in[i/k, (i+1)/k)$, and so
	\begin{align*}
		\int_0^1(g_k(t)-g(t))^2\dd t
		\le\sum_{i=0}^{k-1}\frac1k
		\left(g\left(\frac{i+1}k\right)-g\left(\frac{i}k\right)\right)^2
		\le \frac{g(1)}k\sum_{i=0}^{k-1}
		\left(g\left(\frac{i+1}k\right)-g\left(\frac{i}k\right)\right)
		=\frac{g(1)^2}k
	\end{align*}
	holds because $g$ is increasing.
	Take $f^+_k, f^-_k$ similarly, and $f_k:=a_0+f^+_k-f^-_k$ satisfies
	\[
		f_k\in J_{2k}(C),
		\quad
		\|f_k-f\|_{L^2}^2\le \frac{2C^2}k,
	\]
	and the proof is complete.
	We can of course take $f_k$ directly only from $f$,
	but we have chosen an easier argument.


\subsection{Proof of Theorem \ref{thm:weak_lower_bound}}\label{ap-3}
	Since a smaller set makes the minimax risk smaller,
	it suffices to consider
	\[
		\mathcal{I}:=\{c\phi\mid |c|\le C\}\subset\mathcal{I}_\Phi^0.
	\]
	For simplicity, assume $C\ge1$
	(otherwise, take the functions and constants appearing below to be smaller at that rate).
	Define for each $n=1,2,\ldots$
	\[
		f^+_n:=\frac1{2\sqrt{n}}\phi, \quad
		f^-_n:=-\frac1{2\sqrt{n}}\phi.
	\]
	Then, by applying the argument appearing in \acite{yang1999information},
	we have
	\begin{align}
		\inf_{\hat{f}}\sup_{f\in\mathcal{I}}
		\mathrm{P}_f\left(\|f-\hat{f}\|_{L^2}\ge\frac1{2\sqrt{n}}\right)
		&\ge \inf_{\hat{f}}\sup_{f\in\{f^+_n, f^-_n\}}
		\mathrm{P}_f\left(\|f-\hat{f}\|_{L^2}\ge\frac1{2\sqrt{n}}\right)\nonumber\\
		&\ge\inf_{\hat{f}}\sup_{f\in\{f^+_n, f^-_n\}}
		\mathrm{P}_f(\tilde{f}\ne f),
		\label{eq:sp6}
	\end{align}
	where $\tilde{f}$ is the closer of the two ($f^+_n, f^-_n$) to $\hat{f}$.
	Following the argument in \acite[Proposition 2.1]{tsybakov2008}, we have
	for any $t>0$
	\begin{align}
		\mathrm{P}_{f^+_n}\left(\tilde{f}\ne f^+_n\right)
		&=\mathrm{E}_{f^-_n}\left[
			1_{\{ \tilde{f}=f^-_n \}}(Z^n)
		\frac{\dd\mathrm{P}_{f^+_n}}{\dd\mathrm{P}_{f^-_n}}(Z^n)\right]\nonumber\\
		&\ge t\mathrm{P}_{f^-_n}\left(
			\tilde{f}=f^-_n,\ \frac{\dd\mathrm{P}_{f^+_n}}{\dd\mathrm{P}_{f^-_n}}(Z^n)\ge t\
		\right)\nonumber\\
		&\ge t\left(
			\mathrm{P}_{f^-_n}(\tilde{f}=f^-_n)
			-\mathrm{P}_{f^-_n}\left(
			\frac{\dd\mathrm{P}_{f^+_n}}{\dd\mathrm{P}_{f^-_n}}(Z^n)<t
		\right)\right).
		\label{eq:sp7}
	\end{align}
	Here, $Z^n$ denotes the i.i.d. sequence $(X_i, Y_i)_{i=1}^n$, and
	$\dd\mathrm{P}_{f^+_n}/\dd\mathrm{P}_{f^-_n}$ represents the Radon--Nikodym derivative.
	Then we have, by (\ref{eq:sp6}) and (\ref{eq:sp7}),
	\begin{align}
		\inf_{\hat{f}}\sup_{f\in\mathcal{I}}\mathrm{P}_f\left(
		\|f-\hat{f}\|_{L^2}\ge\frac1{2\sqrt{n}}\right)
		&\ge\frac1{1+t}\left( t\mathrm{P}_{f^-_n}\left(\tilde{f}\ne f^-_n\right)
			+\mathrm{P}_{f^+_n}\left(\tilde{f}\ne f^+_n\right)
		\right)\nonumber\\
		&\ge\frac{t}{1+t}\left(1-\mathrm{P}_{f^-_n}\left(
			\frac{\dd\mathrm{P}_{f^+_n}}{\dd\mathrm{P}_{f^-_n}}(Z^n)<t
		\right)\right)\nonumber\\
		&=\frac{t}{1+t}\mathrm{P}_{f^-_n}\left(
			\frac{\dd\mathrm{P}_{f^+_n}}{\dd\mathrm{P}_{f^-_n}}(Z^n)\ge t
		\right).
		\label{eq:sp8}
	\end{align}
	
	When $f^\circ=f^-_n$ holds,
	the Radon--Nikodym derivative appearing in (\ref{eq:sp8}) can be explicitly written as
	\begin{align}
		\frac{\dd\mathrm{P}_{f^+_n}}{\dd\mathrm{P}_{f^-_n}}(Z^n)
		&=\exp\left(
			\frac1{2\sigma^2}\sum_{i=1}^n(Y_i-f^-_n(X_i))^2
			-\frac1{2\sigma^2}\sum_{i=1}^n(Y_i-f^+_n(X_i))^2
		\right)\nonumber\\
		&=\exp\left(
		\frac1{2\sigma^2}\sum_{i=1}^n\bigl(\xi_i^2-(\xi_i+f^-_n(X_i)-f^+_n(X_i))^2\bigr)
		\right)\nonumber\\
		&=\exp\left(
		\frac2{2\sigma^2\sqrt{n}}\sum_{i=1}^n\xi_i\phi(X_i)
		-\frac1{2\sigma^2n}\sum_{i=1}^n\phi(X_i)^2
		\right).
		\label{eq:sp9}
	\end{align}
	Let us consider the right-hand side of (\ref{eq:sp9}).
	First, $\sum_{i=1}^n\xi_i\phi_1(X_i)$ is a sum of independent symmetric random variables,
	and so the sum itself is also symmetric ($X$ is {\it symmetric} if $-X$ has
	the same distribution as $X$), and thus we have
	$\mathrm{P}\left(\sum_{i=1}^n\xi_i\phi_1(X_i)\ge0\right)\ge1/2$.
	Second, for any $s>0$, by Markov's inequality,
	\[
		\mathrm{P}\left(\frac1{2\sigma^2n}\sum_{i=1}^n\phi(X_i)^2\ge s\right)
		\le \frac1s\mathrm{E}\left[
			\frac1{2\sigma^2n}\sum_{i=1}^n\phi(X_i)^2
		\right]
		=\frac1{2\sigma^2s}
	\]
	holds.
	We determine $s=2/\sigma^2$ to obtain the evaluation
	\[
		\mathrm{P}\left(-\frac1{2\sigma^2n}\sum_{i=1}^n\phi_1(X_i)^2\ge -\frac2{\sigma^2}\right)
		\ge \frac34.
	\]
	Finally, we have
	\begin{align}
		\mathrm{P}
		\left(
		\frac2{2\sigma^2\sqrt{n}}\sum_{i=1}^n\xi_i\phi_1(X_i)
		-\frac1{2\sigma^2n}\sum_{i=1}^n\phi(X_i)^2\ge -\frac2{\sigma^2}
		\right)\ge\frac14.
		\label{eq:sp10}
	\end{align}
	
	By (\ref{eq:sp8})--(\ref{eq:sp10}) and letting $t=e^{-2/\sigma^2}$,
	we have
	\begin{align*}
		\inf_{\hat{f}}\sup_{f\in\mathcal{I}}\mathrm{P}_f\left(
		\|f-\hat{f}\|_{L^2}\ge\frac1{2\sqrt{n}}\right)
		\ge\frac{e^{-2/\sigma^2}}{1+e^{-2/\sigma^2}}\cdot\frac14
		\ge\frac18e^{-2/\sigma^2},
	\end{align*}
	and so we finally obtain the evaluation
	\[
		\inf_{\hat{f}}\sup_{f\in\mathcal{I}}\mathrm{E}_f\left[
		\|f-\hat{f}\|_{L^2}^2\right]\ge\frac{e^{-2/\sigma^2}}{32n},
	\]
	and this is the desired result.


\subsection{Proof of Lemma \ref{lem:sp1}}\label{lem:rev_up_low}

	First, we prove the lower-bound part,
	partially using the proofs in \acite{donoho1996unconditional}.
	When we consider the covering entropy,
	we have only to consider the coefficients, as $\phi=(\phi_i)_{i=1}^\infty$
	is an orthonormal set.
	Thus, define $\mathcal{A}\subset\ell^2$ as
	\begin{align}
		\mathcal{A}:=
		\left\{
			a\in\ell^2
			\,\middle|\,
			\|a\|_{w\ell^p}\le C_1,\ 
			\sum_{i=m+1}^\infty a_i^2\le C_2m^{-\beta},\ 
			m=1,2,\ldots
		\right\}.
		\label{eq:spa}
	\end{align}
	Then it suffices to evaluate $V(\ve):=V_{(\mathcal{A}, \|\cdot\|_{\ell^2})}(\ve)$.
	
	For each $k=1,2,\ldots$, let $a^{(k)}\in\ell^2$ be defined as
	\[
		a^{(k)}:=(
			\underbrace{C_1k^{-1/p}, \ldots, C_1k^{-1/p}}_{k\ \text{times}},
			0, 0, \ldots).
	\]
	Then $\|a^{(k)}\|_{w\ell^p}=C_1$ holds.
	Let us consider the second condition.
	For $1\le m\le k$,
	\begin{align}
		m^\beta\sum_{i=m+1}^\infty (a^{(k)}_i)^2
		=m^\beta(k-m)(C_1k^{-1/p})^2
		\label{eq:sp2}
	\end{align}
	holds.
	Since $x^\beta(k-x)$ is maximized over $x\in[0, k]$ at $x=\frac{\beta}{1+\beta}k$,
	the left-hand side of (\ref{eq:sp2}) is bounded independent of $m$ as
	\begin{align*}
		\sup_{m\ge1}m^\beta\sum_{i=m+1}^\infty (a^{(k)}_i)^2
		\le
		C_1^2\frac{\beta^\beta}{(1+\beta)^{1+\beta}}k^{1+\beta-2/p}
		\le
		C_1^2\frac{\beta^\beta}{(1+\beta)^{1+\beta}},
	\end{align*}
	where we have used the assumption $\beta\le2\alpha$ for the latter inequality.
	If we define a constant
	\[
		C:=\min\left\{1, \frac{C_2^{1/2}(1+\beta)^{(1+\beta)/2}}{C_1\beta^{\beta/2}}\right\},
	\]
	then each $Ca^{(k)}$ is an element of $\mathcal{A}$.
	For each $k$, let us consider a hyperrectangle defined as
	\[
		\mathcal{A}_k:=\{a\in\ell^2\mid |a_i|\le a^{(k)}_i,\ i=1,2,\ldots\}.
	\]
	Obviously, each $\mathcal{A}_k$ is a subset of $\mathcal{A}$ (this actually is based on
	the fact that $\phi$ is an unconditional basis of $\mathcal{I}_\phi^p$),
	and so we have $V(\ve)\ge V_{(\mathcal{A}_k, \|\cdot\|_{\ell^2})}(\ve)$.
	For each pair of distinct vertices of $\mathcal{A}_k$ (which has $2^k$ vertices),
	the $\ell^2$ distance between the two is at least $CC_1k^{-1/p}$,
	and so, by setting $\delta=k^{-1/p}$ in Lemma \ref{lem:sp2},
	we have, for $A$ appearing in the lemma,
	\[
		V\left(\frac{CC_1}2k^{1/2-1/p}\right)\ge
		V_{(\mathcal{A}_k, \|\cdot\|_{\ell^2})}\left(\frac{CC_1}2k^{1/2-1/p}\right)
		\ge Ak
	\]
	for each $k$.
	If we write $C'=CC_1/2$,
	then we have $V(C'k^{-\alpha})\ge Ak$.
	Therefore, for $\ve\in[C'2^{-(j+1)\alpha}, C'2^{-j\alpha}]$,
	\begin{align*}
		V(\ve)\ge V(C'2^{-j\alpha})
		\ge
		2^jA
		=
		2^{j+1}\frac{A}2\ge\frac{A}2\left(\frac\ve{C'}\right)^{-1/\alpha}
		=c\ve^{-1/\alpha}
	\end{align*}
	holds, where $c:=AC'^{1/\alpha}/2$.
	This evaluation holds only for $0<\ve\le C'$;
	however,
	we have $V(\ve)\ge1$ for $\ve>C'$, and hence $V(\ve)\ge C'^{1/\alpha}\ve^{-1/\alpha}$ holds.
	Thus, we have reached the desired result.

	Then, we deal with the upper-bound part.
	By the same logic as in the proof of lower-bound, it suffices to evaluate
	the metric entropy of $\mathcal{A}$ in (\ref{eq:spa}).
	Let $V(\ve):=V_{(\mathcal{A}, \|\cdot\|_{\ell^2})}(\ve)$ similarly.
	For an arbitrary element $a\in\mathcal{A}$,
	let $i_j$ be the index of the term of $a$ having the $j$-th largest absolute value;
	i.e.,
	\[
		|a_{i_1}|\ge|a_{i_2}|\ge\cdots,
	\]
	which is a permutation of $(|a_i|)_{i=1}^\infty$.
	By (\ref{eq:sp1}) and the definition of $\mathcal{A}$,
	\[
		\sum_{j=k+1}^\infty a_{i_j}^2
		\le\sum_{j=k+1}^\infty C_1j^{-2/p}
		\le\int_k^\infty C_1x^{-2/p}\dd x
		=\frac{C_1}{2\alpha}k^{-2\alpha}
	\]
	holds for each $k$.
	Also, by the second condition of $\mathcal{A}$,
	\[
		\sum_{i\ge k^{2\alpha/\beta}+1}a_i^2
		\le C_2k^{-\beta\cdot \frac{2\alpha}\beta} = C_2k^{-2\alpha}
	\]
	holds.
	Thus, if we define
	$\tilde{a}:=(a_1, \ldots, a_{\lceil k^{2\alpha/\beta}\rceil}, 0, 0, \ldots)$
	and $\tilde{i}_1, \tilde{i}_2, \ldots$ similarly,
	\[
		b:=(b_i)_{i=1}^\infty, \quad
		b_i:=\begin{cases}
			\tilde{a}_i\ (=a_i)& \text{if}\ i\in\{\tilde{i}_1,\ldots, \tilde{i}_k\}\\
			0 & \text{otherwise}
		\end{cases},
	\]
	satisfies $\|a-b\|_{\ell^2}\le\sqrt{C_1/2\alpha+C_2}\cdot k^{-\alpha}$.
	Then its quantization
	\[
		\tilde{b}:=\left(
			\sgn(b_i)\frac{\lfloor
				k^{1/2+\alpha}|b_i|\rfloor}{k^{1/2+\alpha}}
		\right)_{i=1}^\infty
	\]
	satisfies $\|b-\tilde{b}\|_{\ell^2}\le k^{-\alpha}$ as $b$ has at most $k$ nonzero terms.
	Since $|b_i|\le C_1$,
	the number of values possibly taken by $\tilde{b}_i$ is at most
	$2C_1k^{1/2+\alpha}+1$.
	Hence, the logarithm of the number of such $\tilde{b}$ values can be upper-bounded by
	\begin{align}
		\log\left(\binom{\lceil k^{2\alpha/\beta}\rceil}{k}
		(2C_1k^{1/2+\alpha}+1)^k\right)
		&\le \frac{2\alpha}\beta k\log k + \left(\frac12+\alpha\right)k\log k+ k\log(2C_1+1)
		\nonumber\\
		&\le C_0k(\log k+1),
		\label{eq:sp-3}
	\end{align}
	where $C_0>0$ is a constant.
	Since $\|a-\tilde{b}\|_{\ell^2}\le(1+\sqrt{C_1/2\alpha+C_2})k^{-\alpha}$ holds,
	if we take $k\sim\ve^{-1/\alpha}$ in the same way as used in the proof of lower-bound,
	then we reach the conclusion.

\begin{rem}
	In the case in which $\beta\ge2\alpha$ holds,
	we can obtain a more accurate bound by using Stirling's approximation.
	However,
	we can see that such a case no longer requires the concept of weak $\ell^p$ norms,
	or else its conditions are too strong.
	Therefore, we have not treated this case.
\end{rem}


\subsection{Proof of Theorem \ref{linear_not_efficient}}\label{ap-35}
	This proof is a refinement of the proof of Theorem 1 in \acite{zhang2002wavelet}.
	First, we prove the following lemma.
	\begin{lem}\label{lem:kantan}
		Let $\gamma\in (0, 1)$, and let $m$ be a positive integer such that
		$m\le n^\gamma \le 2m$.
		For $k=0,\ldots, m-1$,
		let $A_k$ be the number of $X_1, \ldots, X_n$ contained in $[k/m, (k+1)/m)$.
		Then there exists a constant $c=c(\gamma)>0$ such that
		\[
			\mathrm{P}\left(
				\max_{0\le k\le m-1}A_k\ge \frac{cn}{m}
			\right)\le 2^{-n^{1-\gamma}}
		\]
		holds for any $m, n$ satisfying the condition.
	\end{lem}

	\begin{proof}
		For a fixed $k$, $A_k$ can be written as
		$A_k=\sum_{j=1}^n\eta_j$, where $(\eta_j)_{j=1}^n$ is an i.i.d. sequence
		with $\mathrm{P}(\eta_j=0)=1/m$ and $\mathrm{P}(\eta_j=1)=1-1/m$.
		Then, by Chernoff's inequality,
		we have for $t>0$
		\begin{align*}
			\mathrm{P}\left(A_k\ge\frac{cn}m\right)
			\le e^{-cnt/m}\mathrm{E}\left[e^{A_kt}\right]
			=e^{-cnt/m}\mathrm{E}\left[e^{\eta_1t}\right]^n
			=e^{-cnt/m}\left(1-\frac1m+\frac1me^t\right)^n.
		\end{align*}
		Setting $t=\log2$ and assuming $c>\log2$,
		we obtain
		\[
			\mathrm{P}\left(A_k\ge \frac{cn}m\right)
			\le 2^{-cn/m}\left(1+\frac1m\right)^n
			\le (2^{-c}e)^{n/m}\le (2^{-c}e)^{2n^{1-\gamma}},
		\]
		where we have used the fact that $(1+1/x)^x$ is increasing on $x>0$.
		Then we finally have
		\[
			\mathrm{P}\left(
				\max_{0\le k\le m-1}A_k\ge \frac{cn}m
			\right)\le
			m(2^{-c}e)^{2n^{1-\gamma}}\le n^\gamma(2^{-c}e)^{2n^{1-\gamma}}
			=2^{-n^{1-\gamma}}\cdot n^\gamma\left(2^{-(2c-1)}e^2\right)^{n^{1-\gamma}}.
		\]
		Considering the logarithm of the last term,
		it is sufficient to take $c$ as large enough to satisfy
		\[
			(2c-1)\log2\ge 2+\gamma\max_{n\ge1}\frac{\log n}{n^{1-\gamma}},
		\]
		and we obtain the conclusion.
	\end{proof}
	
	We now prove Theorem \ref{linear_not_efficient}.
	Let $\gamma=\frac1{1+\beta}$, and let
	\[
		R^*:=\inf_{\hat{f}:\rm{linear}}\sup_{f^\circ\in\mathcal{J}^p_\psi(C_1, C_2, \beta)}
		\mathrm{E}\left[\|\hat{f}-f^\circ\|_{L^2}^2\right].
	\]
	Fix a linear estimator $\hat{f}(x)=\sum_{i=1}^nY_i\phi_i(x;X^n)$.
	For any $f^\circ\in BV(C)$,
	we have by Fubini's theorem
	\begin{align}
		R^*&\ge\mathrm{E}\left[\|\hat{f}-f^\circ\|_{L^2}^2\right]\nonumber\\
		&=\mathrm{E}\left[
			\left\|\sum_{i=1}^nf^\circ(X_i)\phi_i(\cdot;X^n)-f^\circ\right\|_{L^2}^2
		\right]
		+\sigma^2\sum_{i=1}^n\mathrm{E}\left[\|\phi_i(\cdot;X^n)\|_{L^2}^2\right].
		\label{eq:zhang1}
	\end{align}
	Take a sufficiently large $n$,
	and let $m$ be a power of $2$ in $[\frac12n^{\gamma}, n^{\gamma}]$.
	Notice that it holds that $m\le n^\gamma\le 2m$.
	Then there exists an integer $0\le k<m$ such that
	\begin{align}
		\int_{k/m}^{(k+1)/m}\mathrm{E}\left[
			\sum_{i=1}^n\phi_i(x;X^n)^2
		\right]\dd x\le\frac{R^*}{\sigma^2m}.
		\label{eq:zhang2}
	\end{align}
	Let $f^\circ=Fm^{-\beta/2}\cdot m^{1/2}\psi(m\cdot -k)$,
	where $F:=\min\{C_1, C_2^{1/2}\}$.
	Then we have $f^\circ\in \mathcal{J}^p_\psi(C_1, C_2, \beta)$.
	Let $A$ denote an
	event assured to have a probability of at least $1-2^{-n^{1-\gamma}}$ in Lemma \ref{lem:kantan},
	and we obtain
	\begin{align}
		\int_{k/m}^{(k+1)/m}
		&\mathrm{E}\left[\left(\sum_{i=1}^nf^\circ(X_i)\phi_i(x;X^n)\right)^2,A\right]\dd x
		\nonumber\\
		&\le \int_{k/m}^{(k+1)/m}\mathrm{E}\left[
			\left(\sum_{i=1}^nf^\circ(X_i)^2\right)\left(\sum_{i=1}^n\phi_i(x;X^n)^2\right),A
		\right]\dd x\nonumber\\
		&\le M\cdot\frac{cn}m(Fm^{(1-\beta)/2}\|\psi\|_\infty)^2
		\int_{k/m}^{(k+1)/m}\mathrm{E}\left[
			\sum_{i=1}^n\phi_i(x;X^n)^2,A
		\right]\dd x
		\le\frac{McF^2\|\psi\|_\infty^2}{\sigma^2}\cdot\frac{R^*n}{m^{1+\beta}},
		\label{eq:zhang3}
	\end{align}
	where 
	$M$ is the number of sections $[\ell, \ell+1)$ such that
	$\ell$ is an integer and $[\ell, \ell+1)\cap \mathrm{supp}(\psi)$ is not empty.
	The last inequality has been derived by (\ref{eq:zhang2}).
	By (\ref{eq:zhang1}), (\ref{eq:zhang3}), and the triangle inequality,
	\begin{align*}
		\sqrt{R^*}&\ge\left(\int_{k/m}^{(k+1)/m}\mathrm{E}\left[
			\left(\sum_{i=1}^nf^\circ(X_i)\phi_i(x;X^n)-f^\circ(x)\right)^2, A
		\right]\dd x\right)^{1/2}\\
		&\ge
		\left(\mathrm{P}(A)\|f^\circ\|_{L^2}^2\right)^{1/2}
		-\left(
			\mathrm{E}\left[\left(\sum_{i=1}^nf^\circ(X_i)\phi_i(x;X^n)\right)^2,A\right]\dd x
		\right)^{1/2}\\
		&\ge(1-2^{-n^{1-\gamma}})^{1/2}Fm^{-\beta/2}
		-\left(\frac{McF^2\|\psi\|_\infty^2}{\sigma^2}\right)^{1/2}
		m^{-(1+\beta)/2}n^{1/2}\sqrt{R^*}.
	\end{align*}
	Since $n$ is sufficiently large, we can assume $1-2^{-n^{1-\gamma}}\ge 1/2$.
	Define a constant $G$ by
	\[
		G:=\left(\frac{McF^2\|\psi\|_\infty^2}{\sigma^2}\right)^{1/2}.
	\]
	We have $m^{1+\beta}=m^{1/\gamma}\ge(\frac12n^\gamma)^{1/\gamma}=2^{-1/\gamma}$ by assumption,
	and so we obtain
	\[
		R^*\ge\frac{(F^2/2)m^{-\beta}}{(1+Gm^{-(1+\beta)/2}n^{1/2})^2}
		\ge \frac{F^2}{2(1+2^{1/\gamma}G)^2}m^{-\beta}
		\ge \frac{F^2}{2(1+2^{1/\gamma}G)^2}n^{-\frac\beta{1+\beta}}
	\]
	as desired.


\subsection{Proof of Lemma \ref{cov_eval}}\label{ap-4}
	(mainly following \acite{suzuki2018adaptivity})
	For $f\in\N(L, S, D, B)$ expressed as
	\[
		f=W_{L+1}\circ\rho(W_L\cdot-v_L)\circ\cdots\circ\rho(W_1\cdot-v_1),
	\]
	let us define
	\begin{align*}
		\mathcal{A}_k(f):=\rho(W_{k-1}\cdot-v_{k-1})\circ\cdots\circ\rho(W_1\cdot-v_1),\quad
		\mathcal{B}_k(f):=W_{L+1}\circ\rho(W_L\cdot-v_L)\circ\cdots\circ\rho(W_k\cdot-v_k)
	\end{align*}
	for $k=1,\ldots, L+1$,
	where $\mathcal{A}_1(f)$ denotes the identity map, and $\mathcal{B}_{L+1}(f)=W_{L+1}$.
	Then $f=\mathcal{B}_{k+1}(f)\circ\rho(W_k\cdot-v_k)\circ\mathcal{A}_k(f)$ holds
	for $k=1,\ldots, L$.
	Here, notice that for each $x\in [0, 1]^d$ and $1\le k\le L+1$,
	\begin{align}
		\|\mathcal{A}_k(f)(x)\|_\infty
		&\le D\|W_{k-1}\|_\infty\|\mathcal{A}_{k-1}(f)(x)\|_\infty+\|v_{k-1}\|_\infty \nonumber\\
		&\le DB\|\mathcal{A}_{k-1}(x)\|_\infty+B\nonumber\\
		&\le B+DB^2+D^2B^3+\cdots+D^{k-2}B^{k-1}+D^{k-1}B^{k-1}\nonumber\\
		&\le B^{k-1}(D+1)^{k-1}
		\label{eq:eval_a}
	\end{align}
	holds, where we have used the assumption $B\ge1$ at the last inequality.
	Also, the Lipschitz continuity of $\mathcal{B}_k(f)$ can be derived as
	\begin{align}
		\|\mathcal{B}_k(f)(x)-\mathcal{B}_k(f)(x')\|_\infty
		\le (BD)^{L-k+2}\|x-x'\|_\infty.
		\label{eq:eval_b}
	\end{align}
	
	Let $\ve>0$.
	Suppose $f, g\in\N(L, S, D, B)$ satisfy
	\begin{align*}
		f=W_{L+1}\circ\rho(W_L\cdot-v_L)\circ\cdots\circ\rho(W_1\cdot-v_1),\quad
		g=W_{L+1}'\circ\rho(W_L'\cdot-v_L')\circ\cdots\circ\rho(W_1'\cdot-v_1')
	\end{align*}
	and $\|W_i-W'_i\|_\infty\le\ve$, $\|v_i-v_i'\|_\infty\le\ve$ for each $i$.
	Then we have by (\ref{eq:eval_a}) and (\ref{eq:eval_b})
	\begin{align*}
		|f(x)-g(x)|
		&\le 
		\sum_{k=1}^{L}
		|\mathcal{B}_{k+1}(f)\circ\rho(W_k\cdot-v_k)\circ\mathcal{A}_k(g)(x)
		-\mathcal{B}_{k+1}(f)\circ\rho(W_k'\cdot-v_k')\circ\mathcal{A}_k(g)(x)|\\
		&\quad+|(W_{L+1}-W_{L+1}')\mathcal{A}_{L+1}(g)(x)|\\
		&\le\sum_{k=1}^L(BD)^{L-k+1}\left(\ve D(B(D+1))^{k-1}+\ve\right)+\ve DB^L(D+1)^L\\
		&\le \ve(L+1)B^L(D+1)^{L+1}.
	\end{align*}
	Therefore, for a fixed sparsity pattern,
	and letting $\ve=\left((L+1)B^L(D+1)^{L+1}\right)^{-1}\delta$,
	the $\delta$-covering number is bounded by
	\[
		\left(\frac{2B}\ve\right)^S=\delta^{-S}\left(2(L+1)B^{L+1}(D+1)^{L+1}\right)^S.
	\]
	The number of such patterns is bounded by $\binom{(D+1)^{L}}{S} \le (D+1)^{LS}$,
	and so the $\delta$-covering entropy is bounded by
	\begin{align}
		&\log\left((d+1)^{LS}\delta^{-S}\left(2(L+1)B^{L+1}(D+1)^{L+1}\right)^S\right)
		\label{eq:eval_c}\\
		&=S\log\left( 2\delta^{-1}(L+1)(D+1)^{2L+1}B^{L+1}\right)
		\le 2S(L+1)\log\left(
			\frac{B(L+1)(D+1)}\delta
		\right),\label{eq:eval_d}
	\end{align}
	as desired.


\subsection{Proof of Theorem \ref{thm:main_p>0}}\label{proof_of_main_theorem}

	Let $N$ be an integer in $[n^{\frac1{2\alpha+1}}, 2n^{\frac1{2\alpha+1}}]$.
	Also, suppose that we have an integer $m$ in $[\frac{2\alpha}{\beta(2\alpha+1)}\log_2 n,$
	$\frac{4\alpha}{\beta(2\alpha+1)}\log_2 n]$ (for sufficiently large $n$).
	Fix the target function
	\[
		f^\circ
		=\sum_{(k,\ell)\in T_0}a_{k, \ell}\psi_{k,\ell}.
	\]
	Then, let $(k^1, \ell^1), \ldots, (k^N, \ell^N)\in T_0\setminus T_m$ be
	the $N$ (absolutely) largest coefficients
	$a_{k^1, \ell^1}, \ldots, a_{k^N, \ell^N}$.
	Let
	\[
		T:=\{(k^1, \ell^1), \ldots, (k^N, \ell^N)\}.
	\]
	Then we have
	\begin{align}
		\left\|\sum_{(k, \ell)\in T}a_{k, \ell}\psi_{k, \ell}-f^\circ
		\right\|_{L^2}^2
		=\sum_{(k, \ell)\in T_0\setminus T}a_{k, \ell}^2
		&\le \sum_{(k, \ell)\in T_m}a_{k, \ell}^2
		+\sum_{i=N+1}^\infty C_1 (i^{-1/p})^2\nonumber\\
		&\le C_2n^{-\frac{2\alpha}{2\alpha+1}}+C_1\int_N^\infty x^{-2/p}\dd x\nonumber\\
		&=C_2n^{-\frac{2\alpha}{2\alpha+1}}+C_1\frac{1-p}pN^{-2\alpha}
		\le \left(C_2+C_1\frac{1-p}p\right)n^{-\frac{2\alpha}{2\alpha+1}}.
		\label{eq:nn1}
	\end{align}
	
	Next, we approximate $\sum_{(k, \ell)\in T}a_{k, \ell}\psi_{k, \ell}$ by some neural network.
	Now, $(k, \ell)\in T$ implies that 
	\[
		(k, \ell)=\bigl(
		(k_1, \ldots, k_d), (\ell_1, \ldots, \ell_d)
		\bigr)
	\]satisfies $\max_{1\le i\le d}k_i<m$.
	Let $\tilde{\psi}\in\N(L_\ve, S_\ve, D_\ve, B_\ve)$
	satisfy $\|\tilde{\psi}-\psi\|_{L^2}\le\ve$.
	Since we have
	\[
		\psi_{k, \ell}(x_1, \ldots, x_d)=2^{\frac{k_1+\cdots+k_d}2}\psi\left(
			2^{k_1}t_1-\ell_1, \ldots, 2^{k_d}t_d-\ell_d
		\right),
	\]
	we can construct the approximator
	\begin{align}
		\tilde{f}\in\N(L_\ve+2, N(S_\ve+2D_\ve d+d^2+d+1),
		N D_\ve, \max\{B_\ve, 2^m\})
		\label{eq:nn2}
	\end{align}
	in a manner similar to that shown in Fig. \ref{fig:nn1}
	such that
	\begin{align}
		\left\|\tilde{f}-\sum_{(k, \ell)}a_{k, \ell}\psi_{k, \ell}\right\|_{L^2}
		\le \sum_{(k, \ell)\in T}|a_{k, \ell}|\|\tilde{\psi}_{k, \ell}-\psi_{k, \ell}\|_{L^2}
		\le C_1 N\ve
		\le 2C_1\ve n^{\frac1{2\alpha+1}}
		\label{eq:nn3}
	\end{align}
	holds.
	If we determine $\ve=1/n$ and define $\N_F$ by using the set defined in (\ref{eq:nn2}),
	we have, by Lemma \ref{cov_eval} and the assumption,
	\begin{align}
		V_{(\N_F, \|\cdot\|_{L^\infty})}\left(\frac1n\right)
		&\le 2N\left(S_{1/n}+2D_{1/n}d+d^2+d+1\right)(L_{1/n}+3)
		\log\left(
			\max\{B_{1/n}, 2^m\}(L_{1/n}+3)(ND_{1/n}+1)n
		\right)\nonumber\\
		&\le 4n^{\frac1{2\alpha+1}}\left(C_1'(1+2d)\log n+d^2+d+1\right)(C_1'\log n+3)
		\nonumber\\
		&\quad\cdot\left(
			\log\log(C_2'n)+\frac{2\alpha}{\beta(2\alpha+1)}\log n +\log(C_1'\log n +3)
			+\log (C_1'N\log n+1) + \log n
		\right)\nonumber\\
		&\le C'n^{\frac1{2\alpha+1}}(\log n)^3
		\label{eq:nn4}
	\end{align}
	for some constant $C'>0$.
	
	Combining (\ref{eq:nn1}), (\ref{eq:nn2}), (\ref{eq:nn4}), and Theorem \ref{gen_eval},
	we obtain an evaluation
	\begin{align*}
		R(\hat{f}, f^\circ)
		&\le 4\left(
			\left(C_2+C_1\frac{1-p}p\right)^{1/2}n^{-\frac{\alpha}{2\alpha+1}}
			+2C_1n^{-\frac{2\alpha}{2\alpha+1}}
		\right)^2+C''\left(
			C'(F^2+\sigma^2)n^{-\frac{2\alpha}{2\alpha+1}}(\log n)^3
			+\frac{F+\sigma}n
		\right)\\
		&\le CF^2n^{-\frac{2\alpha}{2\alpha+1}}(\log n)^3
	\end{align*}
	for some $C>0$,
	where $\hat{f}$ denotes the empirical risk minimizer over $\N_F$.


\subsection{Proof of Theorem \ref{p_share_own}}\label{ap-5}
	We consider functions expressed as $f^N=\sum_{i=1}^Nc_if(A_i\cdot-b_i)\in\N^N(L, S, D, B)$,
	where $f\in\N(L, S, D, B)$.
	Then it suffices to consider $f$ defined over $[-B(d+1), B(d+1)]^d
	\subset[-B(D+1), B(D+1)]^d$.
	This changes evaluation (\ref{eq:eval_a}) to
	$\|\mathcal{A}_k(f)(x)\|_\infty\le B^k(D+1)^k$,
	and so $B^{L+1}(D+1)^{L+1}$ in (\ref{eq:eval_c}) is replaced by $B^{L+2}(D+1)^{L+2}$
	in the evaluation of the covering entropy with the domain limited to $[-B(D+1), B(D+1)]^d$.
	However, the upper bound (\ref{eq:eval_d}) is still valid if $L\ge2$,
	and so we have a set $\N_\ve$ for each $0<\ve<1$ satisfying
	\begin{itemize}
		\item
			$\N_\ve$ is an $\ve$-covering of $\N(L, S, D, B)$ with
			respect to $\|\cdot\|_{L^\infty}$, where the domain is limited to
			$[-B(D+1), B(D+1)]^d$;
		\item
			$\log|\N_\ve|\le\displaystyle 2S(L+1)\log
			\left(\frac{B(L+1)(D+1)}\ve\right)$.
	\end{itemize}
	For $x\in[0, 1]^d$,
	we have
	$|f(A_ix-b_i)|\le B^{L+2}(D+1)^{L+2}$ by an evaluation similar to the one in (\ref{eq:eval_a}).
	Also, we have the Lipschitz continuity
	$|f(x)-f(x')|\le B^{L+1}D^{L+1}\|x-x'\|_\infty$, similar to (\ref{eq:eval_b}).
	
	Let us consider two functions in $\N^N(L, S, D, B)$ expressed as
	\begin{align*}
		f^N=\sum_{i=1}^Nc_if(A_i\cdot-b_i),\quad g^N=\sum_{i=1}^Nc_i'g(A_i'\cdot-b_i'),
		\quad
		f, g\in\N(L, S, D, B),
	\end{align*}
	such that
	$\|f-g\|_{L^\infty([-B(D+1), B(D+1)]^d)}\le\ve$ holds
	and $\|A_i-A_i'\|_\infty, \|b_i-b_i'\|_\infty, |c_i-c_i'|\le\ve$ holds for each $i$.
	Then we have for each $x\in[0, 1]^d$
	\begin{align*}
		&|f^N(x)-g^N(x)|\\
		&\le \sum_{i=1}^N|c_if(A_ix-b_i)-c_i'g(A_i'x-b_i')|\\
		&\le \sum_{i=1}^N
			|c_i-c_i'||f(A_ix-b_i)|
			+\sum_{i=1}^N|c_i'||f(A_ix-b_i)-g(A_ix-b_i)|
			+\sum_{i=1}^N|c_i'||g(A_ix-b_i)-g(A_i'x-b_i')|\\
		&\le N\ve B^{L+2}(D+1)^{L+2}+NB\ve +NB\cdot B^{L+1}D^{L+1}(\ve d+\ve)\\
		&\le 3N\ve B^{L+2}(D+1)^{L+2}.
	\end{align*}
	If we determine $\ve=(3NB^{L+2}(D+1)^{L+2})^{-1}\delta$ for $0<\delta<1$,
	the $\delta$-covering entropy
	of $\N^N(L, S, D, B)$
	is now bounded by
	\begin{align*}
		&\log\left(
			\left(\frac{2B}\ve\right)^{N(d^2+d+1)}|\N_\ve|
		\right)\\
		&= N(d^2+d+1)\log\left(
			\frac{2B}\ve
		\right)+\log|\N_\ve|\\
		&\le  N(d+1)^2\log\left(
			\frac{6NB^{L+3}(D+1)^{L+2}}\delta
		\right)
		+2S(L+1)\log\left(\frac{3NB^{L+3}(L+1)(D+1)^{L+3}}\delta\right)\\
		&\le \left(N(d+1)^2+2S(L+1)\right)(L+3)\log\left(\frac{NB(L+1)(D+1)}\delta\right),
	\end{align*}
	as desired.


\subsection{Proof of Lemma \ref{KL-L2}}\label{ap-1}
	The probability law $\mathrm{P}_f$ generated by $f$ is regarded as being on $\R^{d}\times\R$.
	Hence, its density at $z=(x, y)$ is
	\[
		p_f(z)=p_{X}(x)p_{Y\mid X}(y\mid x)
		=\frac1{\sqrt{2\pi\sigma^2}}
		\exp\left(-\frac{(y-f(x))^2}{2\sigma^2}\right).
	\]
	As $p_g$ can be calculated in the same way, we have
	\begin{align*}
		d_\mathrm{KL}(f, g)^2
		=\int_{\R^d\times\R} p_f(z) \log\frac{p_f(z)}{p_g(z)} \dd z
		=\int_{\R^d\times\R}p_f(z) \frac1{2\sigma^2}
		\bigl((y-g(x))^2-(y-f(x))^2\bigr) \dd z.
	\end{align*}
	This coincides with the expectation of
	$\frac1{2\sigma^2}\bigl((Y-g(X))^2-(Y-f(X))^2\bigr)$
	with $Y=f(X)+\xi$.
	The term in parentheses is calculated as 
	\begin{align*}
		\mathrm{E}\left[
			(Y-g(X))^2-(Y-f(X))^2
		\right]
		&=\mathrm{E}\left[(f(X)-g(X)+\xi)^2-\xi^2\right]\\
		&=\mathrm{E}\left[(f(X)-g(X))^2-2\xi(f(X)-g(X))\right]\\
		&=\mathrm{E}\left[(f(X)-g(X))^2\right]
		-2\mathrm{E}[\xi]\cdot\mathrm{E}[f(X)-g(X)]=\|f-g\|_{L^2}^2,
	\end{align*}
	where we have used the facts that each $X$ follows the uniform distribution over $[0, 1]^d$
	and that each $\xi$ is independent of $X$.
	Thus, we obtain the desired result.


\makeatletter

\def\pct{\expandafter\@gobble\string\%}

\immediate\write\@auxout{\pct\space This is a test line.\pct }

\end{document}